\documentclass[preprint]{article}

\usepackage{microtype}
\usepackage{graphicx}
\usepackage{subfigure}
\usepackage{booktabs} 

\usepackage{hyperref}

\usepackage{icml2024}

\usepackage{amsmath}
\usepackage{amssymb}
\usepackage{mathtools}
\usepackage{amsthm}

\usepackage[capitalize,noabbrev]{cleveref}

\usepackage[textsize=tiny]{todonotes}

\usepackage{acronym}

\usepackage[inline]{enumitem}

\usepackage{tikz}
\usetikzlibrary{graphs,decorations.pathmorphing,decorations.markings} 
\usetikzlibrary{arrows}
\newcommand{\X}{\mathcal{X}}

\newcommand{\Y}{\mathcal{Y}}

\renewcommand{\S}{\mathcal{S}}

\newcommand{\Pcal}{\mathcal{P}}

\newcommand{\N}{\mathbb{N}}
\newcommand{\Nstar}{{\mathbb{N}^\star}}
\newcommand{\R}{\mathbb{R}}
\newcommand{\Rcal}{\mathcal{R}}
\renewcommand{\P}{\ensuremath{\mathbb{P}}}
\newcommand{\A}{\mathcal{A}}

\newcommand{\B}{\mathcal{B}}
\newcommand{\E}{\mathbb{E}}
\newcommand{\F}{\ensuremath{\mathcal{F}}}
\renewcommand{\H}{\mathcal{H}}
\newcommand{\G}{\ensuremath{\mathcal{G}}}
\renewcommand{\L}{\mathcal{L}}
\newcommand{\Lfrak}{\mathfrak{L}}
\newcommand{\Id}{\mathrm{Id}}

\newcommand{\Cb}{\ensuremath{C_\mathrm{b}}}
\newcommand{\cvw}{\leadsto}
\renewcommand{\d}{\mathrm{d}}
\newcommand{\goesto}[1]{\xrightarrow[#1]{}}
\newcommand{\scalar}[1]{\left\langle#1\right\rangle}
\newcommand{\norm}[1]{\left\lVert#1\right\rVert}

\newcommand{\Pas}{\ensuremath{\P\text{-\ac{as}}}}
\newcommand{\floor}[1]{\lfloor#1\rfloor}
\newcommand{\ind}{\mathbf{1}}

\newcommand{\mmd}{\mathrm{MMD}}

\newcommand{\K}{\mathcal{K}}
\newcommand{\Q}{\mathcal{Q}}
\renewcommand{\floor}[1]{\left\lfloor#1\right\rfloor}
\newcommand{\dtv}{d_\mathrm{TV}}

\newcommand{\dil}{\mathrm{dil}}

\theoremstyle{plain}
\newtheorem{theorem}{Theorem}[section]
\newtheorem{proposition}[theorem]{Proposition}
\newtheorem{lemma}[theorem]{Lemma}
\newtheorem{corollary}[theorem]{Corollary}
\newtheorem{example}[theorem]{Example}
\theoremstyle{definition}
\newtheorem{definition}[theorem]{Definition}

\theoremstyle{remark}
\newtheorem{remark}[theorem]{Remark}

\newacro{EWC}[EWC]{empirically weakly converging}
\newacroindefinite{EWC}{an}{an}

\newacro{ESC}[ESC]{empirically strongly converging}
\newacroindefinite{ESC}{an}{an}

\newacro{AMS}[AMS]{asymptotically mean stationary}
\newacroindefinite{AMS}{an}{an}

\newacro{WLLNE}[WLLNE]{weak law of large numbers for events}
\newacro{SLLNE}[SLLNE]{strong law of large numbers for events}
\newacro{LLNE}[LLNE]{law of large numbers for events}

\newacro{RKHS}[RKHS]{reproducing kernel Hilbert space}
\newacroindefinite{RKHS}{an}{a}

\newacro{SVM}[SVM]{support vector machine}
\newacroindefinite{SVM}{an}{a}

\newacro{KME}[KME]{kernel mean embedding}
\newacroindefinite{KME}{a}{a}

\newacro{CKME}[CKME]{conditional kernel mean embedding}
\newacroindefinite{CKME}{a}{a}

\newacro{GP}[GP]{Gaussian process}
\newacroindefinite{GP}{a}{a}

\newacro{MMD}[MMD]{maximum mean discrepancy}

\newacro{iid}[i.i.d.]{independent and identically distributed}
\newacroindefinite{iid}{an}{an}

\newacro{as}[a.s.]{almost surely}
\newacroindefinite{as}{an}{an}

\newacro{ae}[a.e.]{almost-everywhere}

\newacro{onb}[ONB]{orthonormal basis}
\newacroindefinite{onb}{an}{an}

\newacro{lhs}[LHS]{left-hand side}
\newacro{rhs}[RHS]{right-hand side}

\usepackage{centernot}

\newcommand{\review}[1]{{\color{blue}#1}}
\renewcommand{\review}[1]{#1}

\icmltitlerunning{
    On the Consistency of Kernel Methods with Dependent Observations
}

\begin{document}

\twocolumn[
    \icmltitle{On the Consistency of Kernel Methods with Dependent Observations}

    
    
    \icmlsetsymbol{equal}{*}
    
    \begin{icmlauthorlist}
    \icmlauthor{Pierre-François Massiani}{DSME}
    \icmlauthor{Sebastian Trimpe}{DSME}
    \icmlauthor{Friedrich Solowjow}{DSME}
    \end{icmlauthorlist}

    \icmlaffiliation{DSME}{Institute for Data Science in Mechanical Engineering, RWTH Aachen University, Aachen, Germany}

    \icmlcorrespondingauthor{Pierre-François Massiani}{massiani@dsme.rwth-aachen.de}

    \icmlkeywords{Kernel methods, Statistical learning theory, Support vector machines} 

    \vskip 0.3in
]



\printAffiliationsAndNotice{}

\begin{abstract}
    The consistency of a learning method is usually established under the assumption that the observations are a realization of an independent and identically distributed (i.i.d.) or mixing process.
    Yet, kernel methods such as support vector machines (SVMs), Gaussian processes, or conditional kernel mean embeddings (CKMEs) all give excellent performance under sampling schemes that are obviously non-i.i.d., such as when data comes from a dynamical system.
    We propose the new notion of \emph{empirical weak convergence (EWC)} as a general assumption explaining such phenomena for kernel methods.
    It assumes the existence of a random asymptotic data distribution and is a strict weakening of previous assumptions in the field.
    Our main results then establish consistency of SVMs, kernel mean embeddings, and general Hilbert-space valued empirical expectations with EWC data.
    Our analysis holds for both finite- and infinite-dimensional outputs, as we extend classical results of statistical learning to the latter case.
    In particular, it is also applicable to CKMEs.
    Overall, our results open new classes of processes to statistical learning and can serve as a foundation for a theory of learning beyond i.i.d. and mixing.
\end{abstract}
\section{Introduction}

A learning method is \emph{consistent} when the learned function is optimal in a certain sense in the infinite-sample limit, and the most common assumption to prove consistency is that training data comes from \iac{iid} process.
That assumption is often blatantly violated, however, leading to drops in performance for many learning algorithms.
For instance, replay buffers were introduced in deep reinforcement learning to mitigate this issue; see~\cite{MKS2013}.
In contrast, kernel methods such as \acp{SVM}, Gaussian processes, or \acp{CKME} seem unaffected by dependencies in the training data and often perform well despite the absence (or limited amount) of theoretical justification; see \cite{BST2020,vRNT2021} for examples.
\par
A common such case is that of learning on a Markov chain, where data is neither independent (previous states influence those that follow) nor identically distributed (transition probabilities may differ in different states).
The standard assumption to get consistency under such sampling is \emph{mixing}\,\cite{SC2009}.
It replaces independence with decaying correlations between samples as their temporal distance increases.
There is a vast literature showing consistency under mixing; examples are\,\cite{SHS2009,SC2009,Irl1997}.
\par
\review{%
A possible explanation for the popularity of the mixing assumption is that it allows dependencies while retaining a central concept of statistical learning: there exists an asymptotic distribution describing future samples.
Learning is then formalized as minimizing the risk of those future samples.
There are currently two main approaches to generalizing statistical learning to non-mixing data: assuming the existence of an asymptotic distribution to define the risk\,\cite{SHS2009}, or defining a non-asymptotic notion of risk\,\cite{simchowitz2018learning,ZT2022}.
We focus on the first approach, as it leverages more standard tools whereas the alternatives impose additional conditions on the data-generating process.
Central to this approach is examining whether the empirical measure has a limit instead of imposing stronger conditions enforcing the existence of this limit and convergence thereto with certain speed --- e.g., \ac{iid}~
This pertains to whether a form of the \emph{\ac{LLNE}} holds.
While \citet{SHS2009} provide first results on learning in this context, their analysis still excludes simple, non-pathological processes such as Example\,\ref{ex:simple EWC} below.
This shows the necessity for a theory of consistency under assumptions weaker than \ac{iid} or mixing.
}
\par
We propose \ac{EWC}\footnote{In what follows, we use the abbreviation \ac{EWC} indistinctly for the noun ``empirical weak convergence'' and the adjective ``empirically weakly converging''.} processes for the basis of such a theory.
They are those that possess an asymptotic data distribution, which may be random.
We focus on the weak limit of the empirical measure for that distribution since it is one of the weakest (and thus more general) and simplest notion of convergence of random measures.
Our main results then establish consistency of \acp{SVM} and of some Hilbert-space valued empirical expectations under \ac{EWC}.
Interestingly, the weak convergence we assume imposes a \emph{continuously differentiable} loss function for consistency, which is more restrictive than the continuity required in \cite{SHS2009} under convergence in total variation.
In other words, the gain of generality allowed by \ac{EWC} processes comes at the price of weaker guarantees.
\par
Our results formalize the general intuitive idea that \emph{consistency is w.r.t. the data distribution effectively generated by the process}.
Albeit simple, it offers the flexibility of a \emph{path-dependent} asymptotic distribution, which is completely new to the best of our knowledge.
Previous standard assumptions such as \ac{iid}, mixing, and the \ac{LLNE} are then sufficient conditions that specify this asymptotic distribution a priori.
\review{%
With this shift of perspective, \ac{EWC} pushes the reasoning of statistical learning to minimize an asymptotic risk beyond those assumptions.
}
\par
All of our results hold in a framework more general than that of standard statistical learning theory of \citet{SC2008}.
Specifically, we allow for a general separable Hilbert output space, whereas the above reference only considers scalar outputs.
As an immediate benefit, our analysis also applies to cases that were so far excluded such as \acp{CKME}, which are an instance of \acp{SVM} with infinite-dimensional outputs in general.

The rest of the paper is organized as follows.
We begin with preliminaries in Section~\ref{sec:preliminaries}.
Section~\ref{sec:ewc} defines \ac{EWC} processes and their properties.
Specifically, Theorems~\ref{thm:characterization ewc as} and \ref{thm:characterization ewc prob} provide necessary and sufficient conditions for \ac{EWC}, and Theorem~\ref{thm:averages of Cb(H)} shows that empirical expectations of Hilbert-space valued continuous and bounded functions converge under \ac{EWC}.
Further, Theorems~\ref{thm:continuous transformation} and~\ref{thm:ewc transition pairs} characterize \ac{EWC} of multivariate processes, which is particularly relevant as it constitutes one of our main assumptions.
Finally, Section~\ref{sec:consistency} states our main results on consistency of kernel methods, Theorems~\ref{thm:consistency kme} and~\ref{thm:consistency svm under ewc}.
With a few exceptions, proofs are in Appendix~\ref{apdx:proofs}.
Appendix~\ref{apdx:random sets} contains technical results on random compact sets necessary for the proof of Theorem~\ref{thm:averages of Cb(H)}, Appendix~\ref{apdx:slt separable hilbert} is the generalization of the general representer theorem of \cite{SC2008} to separable Hilbert output spaces, and Appendix~\ref{apdx:measurability SVMs} justifies the measurability of risks with random measures and of \acp{SVM} with such output spaces, upon which our results rely.
\review{%
We finish with a remark on terminology:
as in \cite{SC2008}, we mean by ``\ac{SVM}'' a generalized \ac{SVM}, that is, regularized empirical risk minimization over \iac{RKHS} with an arbitrary loss.
This differs from the historic definition\,\cite{SS1998}, which only considers the Hinge loss.
}

\subsection{Related work}

\paragraph{Consistency under dependent sampling}
Consistency is a fundamental property a learning method should have.
It has been thoroughly studied for \acp{SVM} in the case of \ac{iid} data with finite-dimensional outputs.
One method to obtain consistency are the so-called \emph{oracle inequalities}, of which a complete exposition is in Chapter 6 of \cite{SC2008}.
The other main method for consistency is the \emph{integral operator} technique, which is limited to the square loss as it leverages the specific form of the solution to the \ac{SVM} \cite{CDV2007}.
While it allows infinite-dimensional outputs, the required assumptions on the operators historically limited the allowed class of kernels in that case.
Recent results show consistency for a broad class of kernels; see \cite{PM2022,LMMG2022}.
Next, many bodies of work examine consistency of \acp{SVM} under \emph{mixing} assumptions rather than \ac{iid} \cite{SHS2009,SC2009,Irl1997}.
Finally, all of these works also provide learning rates quantifying the speed of the convergence.
Such rates also require restricting the class of data-generating distributions; see \cite{SHS2009,VLS2011,CDV2007}.
In the present work, we focus on \emph{whether convergence occurs} and neglect the question of its speed; in fact, we allow arbitrarily slow convergence and leave this question for future work.
Therefore, we rely on a third, non-quantitative method leveraging only laws of large numbers \cite{SC2009}.
Our results generalize those of Section 2 in that reference to \ac{EWC} data and infinite-dimensional outputs.
\review{%
Consistently with the no-free-lunch theorem\,\cite{SHS2009}, we find that the sequence of should be annealed at a rate that depends on the data-generating process in general.
This is in accordance with results on the inconsistency of kernel ridgeless regression\,\cite{RZ2019,Buc2022,BBP2023}, where the regularization is fixed to $0$.
}
\par
There are no further relaxations of independence that are specific to \acp{SVM}, to the best of our knowledge.
Instead, \emph{learnability} focuses on whether learning is possible at all without considering a specific learning method.
Results on learnability are typically on the \emph{existence} of a learning method, and are thus of independent interest.
An overview of recent results in that field is available in \cite{Han2021}.

\paragraph{Learning theory with infinite-dimensional outputs}
The classical framework of statistical learning theory from \citet{SC2008} assumes a scalar output space.
It immediately generalizes to finite-dimensional outputs by reasoning component-wise, but the extension to infinite-dimensional outputs is less clear.
In fact, there is no systematic exposition of the generalization of the definitions or the results of the above reference to that case, with the notable exception of the square loss as discussed above.

A contribution of this work is thus the generalization of the setup of statistical learning theory to separable Hilbert output spaces.
As an immediate consequence, we obtain consistency of \acp{SVM} in that case under appropriate assumptions such as \ac{iid} or \ac{EWC}.
In particular, our analysis applies to \acp{CKME}, which are, when regularized, the solutions of infinite-dimensional \acp{SVM} where the output space is itself \iac{RKHS} and output data consists of kernel partial evaluations; cf. \cite{GLB2012}.
This is of particular interest since, historically, unregularized \acp{CKME} have suffered from debates on meaningful assumptions for a well-posed definition; see \cite{KSS2020,MK2020}.
A way to understand this debate is to notice that \iac{CKME} is the solution of an \emph{unregularized} \ac{SVM}, which becomes problematic when the solution does not exist in the \ac{RKHS} where regression is performed.
We avoid this concern by only considering the regularized problem; consistency is meaningful even when the sequence of regressors does not converge in the \ac{RKHS}.
There is a rich literature on the asymptotic properties of \ac{CKME} estimation.
For instance, \citet{PM2022} and \citet{LMMG2022} show consistency and (optimal) learning rates.
While our theory is not quantitative, it applies in a more general case, as it does not require \ac{iid} data, and we are not limited to the square loss.

\section{Preliminaries and notations}\label{sec:preliminaries}
We introduce in this section necessary definitions to state our results.
We focus on the different spaces we consider, Markov kernels, statistical learning theory, and \acp{RKHS}.

\subsection{Sets and topology} \label{sec:preliminaries:sets}

A Polish space is a topological space $(\X,\mathcal{T})$ that is separable and completely metrizable.
In what follows, we abuse notation and simply say that $\X$ is a Polish space, omitting specifying the topology since it is always clear from context.
Specifically, we equip product spaces with the natural product topology, and any normed vector space with the topology induced by its norm.
In particular, $\R$ is equipped with its usual topology (for which it is Polish), a complete subset of a separable Hilbert space is Polish, and a product of Polish spaces is Polish.
Then, a Polish space is \emph{locally compact} if every point has a compact neighborhood; this is for instance the case of finite-dimensional vector spaces.
Finally, if $\X$ and $\Y$ are Polish, then $C(\X;\Y)$ is the set of continuous functions from $\X$ to $\Y$.
If, additionally, $\Y$ is metric, $\Cb(\X;\Y)$ is the set of continuous bounded functions from $\X$ to $\Y$.
If $\X$ is compact, then $\Cb(\X;\Y)$ equipped with the topology of uniform convergence is separable.

We equip any Polish space $\X$ with its $\sigma$-algebra of Borel sets, that is, the $\sigma$-algebra generated by its topology.
We denote it by $\B(\X)$.
In this context, we define $\Pcal(\X)$ as the set of probability measures on the measurable space $(\X,\B(\X))$.
For any probability measure $P\in\Pcal(\X)$ and $f\in\Cb(\X;\R)$, we define the standard notation \begin{equation*}
    Pf := \int_\X f(x)\d P(x)
\end{equation*}
We also use this notation if $P$ is not a probability measure but only a finite sum or difference thereof, making it a signed measure.
We endow $\Pcal(\X)$ with the topology of weak convergence:
recall that a sequence of probability measures $(P_n)_{n\in\Nstar}\subset\Pcal(\X)$ converges weakly to a probability measure $Q\in\Pcal(\X)$ if, for every $f\in\Cb(\X;\R)$,\begin{equation*}
    P_n\,f\goesto{n\to\infty} Q\,f.
\end{equation*}
We denote this as $P_n\cvw Q$.
Further, it is known that if $\X$ is Polish, then $\Pcal(\X)$ is metrizable; see for instance \cite{SGF2010}.

Finally, if $\X$ is a topological space, $P\in\Pcal(\X)$, and $\G$ is a Banach space, we define $\L_0(\X;\G)$  the set of (Bochner) measurable functions from $\X$ to $\G$.
Then, $\L_\infty(\X;\G)$ and $\L_2(\X,P;\G)$ are its subsets consisting of bounded functions and of functions with finite $2$ norm w.r.t. $P$, $L_2(\X,P,\G)$ is the Bochner space of equivalence classes of functions in $\L_2(\X,P;\G)$, and $L_\infty(\X,P;\G)$ is the set of $P$-essentially bounded functions from $\X$ to $\G$.

\subsection{Markov kernels and random elements}
All the random variables that we consider in this work are defined on a complete, standard Borel probability space $(\Omega,\A,\P)$.
We consider discrete stochastic processes, which we index by the set of positive integers $\Nstar$ without loss of generality unless mentioned otherwise.
If $X$ is a process taking values in $\X$, we consistently use the notation $X_n$ for its value at time $n\in\Nstar$ and do not introduce this notation again later.
We also use the notation $X_{1:n}$ to denote the random $\X^n$-valued vector $(X_1,\dots,X_n)$, $n\in\Nstar$.
Further, if $Z$ is an $\X\times\Y$-valued process, the notation $Z = (X,Y)$ means that $X$ and $Y$ are $\X$- and $\Y$-valued processes, respectively, such that $Z_n = (X_n, Y_n)$ for all $n\in\Nstar$.

For an $\X$-valued random variable $X$ and a sub-$\sigma$-algebra $F\subset\B(\X)$, we denote by $\E[X\mid F]$ and $\P[X\in\cdot\,\mid\,F]$ versions of the conditional expectation and law of $X$ given $F$.
The specific choices of conditional expectation and law do not play a role in this article.
If $F$ is the $\sigma$-algebra generated by one or many random variables, we conveniently replace $F$ with the appropriate variables in the above notations.
Finally, if $Y$ is a $\Y$-valued process, we allow the notation $Y_{1:0}$ when conditioning to denote conditioning w.r.t. the trivial $\sigma$-algebra $\{\emptyset,\Omega\}$.

Next, we introduce Markov kernels \cite{Kal2017}, which we consider in two contexts: stochastic input-output maps, and random measures.
\begin{definition}[Markov kernel]\label{def:markov kernel}
    Let $(\X,\A_\X)$ and $(\Y,\A_\Y)$ be measurable spaces.
    A \emph{Markov kernel} from $\X$ to $\Y$ is a map $p:\A_\Y\times\X\to [0,\infty)$ such that \begin{enumerate}
        \item \label{def:markov kernel:pointwise measure} for all $x\in\X$, $p(\cdot, x)$ is a probability measure on $\Y$;
        \item \label{def:markov kernel:setwise measurability} for all $A\in\A_\Y$, $p(A, \cdot)$ is $\A_\X$-measurable.
    \end{enumerate}
\end{definition}
The following notion generalizes continuity to such kernels.
It is classically defined for continuous processes, but readily applies to discrete ones as well~\cite{li2009criteria}.
\begin{definition}
    Let $\X$ and $\Y$ be Polish spaces.
    We say that a Markov kernel $p$ from $\X$ to $\Y$ is Feller-continuous if the map $x\in\X\mapsto p(\cdot, x)f$ is in $\Cb(\X;\R)$ for any $f\in\Cb(\Y;\R)$.
\end{definition}
\begin{definition}
    Let $\X$ be Polish space.
    We call a measurable map $P:\Omega\to\Pcal(\X)$ a \emph{random measure}.
\end{definition}
An alternative definition is through Markov kernels: a map $P:\Omega\to\Pcal(\X)$ is a random measure if, and only if, the map $p:\B(\X)\times\Omega\to\R$ defined by $p(B,\omega) = P(\omega)(B)$ for all $B\in\B(\X)$ and $\omega\in\Omega$ is a Markov kernel; see Lemma 1.14 in \cite{Kal2017}.
If $P$ is a random measure on $\X$, we abuse notation and often omit its dependency on $\omega\in\Omega$.
The \emph{empirical (probability) measure} of an $\X$-valued process $X$ at time $n\in\Nstar$ is \begin{equation*}
    \eta^{X}_{n} = n^{-1}\sum_{i=1}^{n}\delta_{X_i},
\end{equation*}
where $\delta_x\in\Pcal(\X)$ is the Dirac measure with mass at $x\in\X$.
The empirical measure is a random probability measure.
We also use the notation $\eta_n^X$ when $X$ is a deterministic sequence; then, $\eta_n^X$ is a probability measure.

\subsection{Elements of statistical learning theory}

We generalize here the definitions of \cite{SC2008} to the case of an infinite-dimensional output space, which is necessary for our main result (Theorem~\ref{thm:consistency svm under ewc}) to also encompass \acp{CKME}.
The definitions here are straightforward extensions of those of the above reference.
More interesting is the resulting general representer theorem, which generalizes Theorem 5.9 in the reference to separable infinite-dimensional output spaces.
Although it is new to the best of our knowledge, its proof is a direct generalization of its counterpart in \cite{SC2008}.
For this reason, we defer the statement and proof of this general representer theorem to Appendix~\ref{apdx:slt separable hilbert}.

In what follows, $\X$ is a Polish space, $\G$ is a Hilbert space, and $\Y$ is a complete subset of $\G$.
\begin{definition}
    A function $L:\X\times\Y\times\G\to[0,\infty)$ is called a \emph{loss function} (on $\X\times\Y\times\G$) if it is measurable.
    Then, it is convex (resp. continuous) if, for all $x\in\X,y\in\Y$, the function $L(x,y,\cdot)$ is convex (resp. continuous).
    Further, $L$ is \emph{locally bounded} if, for all bounded $A\subset\G$, the restriction $L_{\mid\X\times\Y\times A}$ is bounded, and \emph{locally Lipschitz continuous} if, for all $a>0$, the following quantity is finite: \begin{equation*}
        \lvert L\rvert_{a,1}:= \sup_{\substack{(t,t^\prime)\in\G^2\\0<\norm{t-t^\prime}_\G \leq 2a}}\sup_{\substack{x\in\X\\y\in\Y}}\frac{\lvert L(x,y,t) - L(x,y,t^\prime)\rvert}{\norm{t-t^\prime}_\G}.
    \end{equation*}
    It is \emph{Lipschitz continuous} if $\lvert L\rvert_1 = \sup_{a>0}\lvert L\rvert_{a,1}<\infty$.
    Finally, the loss is (Fréchet) differentiable if the map $L(x,y,\cdot)$ is Fréchet differentiable on $\G$ for all $(x,y)\in\X\times\Y$, that is, for all $t\in\G$ there exists a vector $A_{x,y,t}\in\G$ such that\begin{equation*}
        \lim_{\substack{h\to 0\\h\neq 0}} \frac{\lvert L(x,y,t+h) - L(x,y,t) - \scalar{A_{x,y,t},h}\rvert}{\norm{h}_\G} = 0.
    \end{equation*}
    The vector $A_{x,y,t}$ is then written $\nabla L(x,y,t)$.
    We say that $L$ is \emph{continuously (Fréchet) differentiable} if the map $(x,y,t)\in\X\times\Y\times\G\mapsto \nabla L(x,y,t)$ is continuous.
\end{definition}
\begin{definition}
    A loss function $L:\X\times\Y\times\G\to\R$ is called a Nemitski loss function if there exists a measurable function $b:\X\times\Y\to[0,\infty)$ and an increasing function $h:[0,\infty)\to[0,\infty)$ such that, for all $(x,y,t)\in\X\times\Y\times\G$, \begin{equation*}
        L(x,y,t) \leq b(x,y) + h(\norm{t}_\G).
    \end{equation*}
    Further, it is a Nemitski loss of order $p\in(0,\infty)$ if there exists a constant $c>0$ such that the above holds with $h(s) = c\cdot s^p$, for $s>0$.
    Finally, we say that $L$ is a $J$-integrable Nemitski loss if $b\in \L_1(\X\times\Y,J;\R)$ for $J\in\Pcal(\X\times\Y)$.
\end{definition}
\begin{definition}
    Let $L$ be a loss function and $J\in\Pcal(\X\times\Y)$.
    For any measurable function $f:\X\to\G$, we define its \emph{$L$-risk} as \begin{equation*}
        \Rcal_{L,J}(f) = \int_{\X\times\Y} L(x,y,f(x))\d J(x,y).
    \end{equation*}
    The \emph{Bayes $L$-risk} is then $\Rcal^\star_{L,J} = \inf\{\Rcal_{L,J}(f)\mid f\in\L_0(\X;\G)\}$.
\end{definition}
\begin{definition}
    Let $L$ be a loss function and $\F$ a set of measurable functions from $\X$ to $\G$.
    For $J\in\Pcal(\X\times\Y)$, we say that $\F$ is $(L,J)$-rich if \begin{equation*}
        \Rcal_{\F,L,J} := \inf_{f\in\F}\Rcal_{L,J}(f) = \Rcal^\star_{L,J}.
    \end{equation*}
\end{definition}
A \emph{learning method} $\Lfrak$ is then a mapping that maps any training set $Z=((x_i,y_i))_{i=1}^n\in(\X\times\Y)^n$, $n\in\Nstar$, to a unique measurable function $f_Z:\X\to\G$.
We postpone introducing the notion of consistency that we consider to Section~\ref{sec:consistency}, since it requires first defining \ac{EWC} processes.

\subsection{Vector-valued RKHSs and SVMs}
In this section, $\X$ is a measurable space and $\G$ a separable Hilbert space.
\begin{definition}
    A $\G$-valued \ac{RKHS} $\H$ on $\X$ is a Hilbert space $(\H, \scalar{\cdot,\cdot}_\H)$ of functions such that for all $x\in\X$, the evaluation operator $S_x:f\in\H\mapsto f(x)\in\G$ is continuous.
    Then, we define $K(\cdot,x) = S_x^\star$ and $K(x, x^\prime) = S_xS_{x^\prime}^\star$, for all $x,x^\prime\in\X$.
    The map $K:\X\times\X\to\L(\G)$ is called the (operator-valued) reproducing kernel of $\H$, where $\L(\G)$ is the Banach space of continuous linear operators on $\G$.
\end{definition}
\begin{theorem}
    Let $\H$ be a $\G$-valued \ac{RKHS} with kernel $K$.
    Then, $K$ is Hermitian, positive semi-definite\footnote{Recall that a bivariate function $\phi:\X\times\X\to\L(\G)$ is Hermitian if $\phi(x,x^\prime) = \phi(x^\prime, x)^\star$. Further, it is positive semi-definite if for all $n\in\Nstar$, $(x_i)_{i=1}^n\in\X^n$, and $(g_i)_{i=1}^n\in\G^n$, $\sum_{i=1}^n\sum_{j=1}^n\scalar{g_i, \phi(x_i,x_j)g_j}_\G \geq 0$.}, and the reproducing property holds for all $x\in\X,\,f\in\H$, and $g\in\G$: \begin{equation*}
        \scalar{f(x), g}_\G = \scalar{f, K(\cdot, x)g}_\H.
    \end{equation*}
\end{theorem}
It is well known that, for every positive semi-definite function $K:\X\times\X\to\L(\G)$, there exists a unique $\G$-valued \ac{RKHS} of which $K$ is the unique reproducing kernel; see \cite{CDVT2006}.
A special case of vector-valued \acp{RKHS} is that of scalar-valued \acp{RKHS}, that is, when $\G = \R$.
Then, we introduce a symmetric, positive definite function $k:\X\times\X\to\R$ such that $K(x,x^\prime) = k(x,x^\prime)\Id_\R$ (which always exists), and refer to $k$ as the reproducing kernel of $\H$.
\begin{definition}
    We say that a kernel $K$ is \emph{bounded} if \begin{equation*}
        \norm{K}_\infty := \sup_{x\in\X}\sqrt{\norm{K(x,x)}_{\L(\G)}} < \infty.
    \end{equation*}
\end{definition}
\par
Given a convex loss $L$, a $\G$-valued \ac{RKHS} $\H$, and a training set $Z=((x_i,y_i))_{i=1}^n\in(\X\times\Y)^n$, $n\in\Nstar$, there exists a unique minimizer to the following minimization problem, where $\lambda>0$ is a regularization parameter (cf. Lemma~\ref{lemma:existence uniqueness SVM solutions}): \begin{equation}\label{eq:SVM}
    f_{Z,\lambda} = \arg\min_{f\in\H} \Rcal_{L,\eta^Z_n}(f) + \lambda \norm{f}_\H^2.
\end{equation}
We call the learning method that maps, given a sequence of regularization parameters $(\lambda_n)_{n\in\Nstar}$, a sequence of training sets $(Z_n)_{n\in\Nstar}$ to the sequence $(f_{Z_n,\lambda_n})$ the $(\lambda_n)$-\ac{SVM}.

\section{Empirical weak convergence} \label{sec:ewc}
In this section, we introduce the notion of \ac{EWC} processes.
We show some first properties and discuss the connections with other standard notions.
Finally, we examine \ac{EWC} of joint processes.

\subsection{Definition and first properties}\label{sec:ewc definition}
We are interested in processes that possess an asymptotic data distribution.
A reasonable definition is to take for such an asymptotic distribution the limit of $(\eta_n^X)_{n\in\Nstar}$ considered as an element of a topological space; here, $\Pcal(\X)$ endowed with a suitable topology.
We choose that of weak convergence, as it is allows for a broad class of processes while still providing strong guarantees, as we will see it.
Further, since $(\eta_n^X)_{n\in\Nstar}$ is a \emph{random} measure, the convergence should be in a probabilistic sense.
\begin{definition} \label{def:ewc}
    Let $\X$ be a Polish space and $d$ be any metric that metrizes the weak convergence on $\Pcal(\X)$.
    We say that an $\X$-valued process $X$ is \emph{\acf{EWC}} \emph{in probability} if there exists a random measure $P$ on $\X$ such that and $\eta_n^X\cvw P$ in probability, that is, \begin{equation}
        \lim_{n\to\infty}d(\eta_n^X,P) = 0, \label{eq:ewc}
    \end{equation}
    where convergence is in probability.
    Further, we say that $X$ is \emph{\ac{EWC} \ac{as}} if the weak convergence $\eta_n^X\cvw P$ occurs \ac{as}, that is, if the convergence in \eqref{eq:ewc} holds \ac{as}
    In either case, the random measure $P$ is called a limit measure of the process.
\end{definition}
\begin{proposition} \label{prop:uniqueness limit}
    Let $\X$ be a Polish space.
    The notions of \ac{EWC} in probability and \ac{EWC} \ac{as} are independent of the metric $d$ that metrizes weak convergence.
    Furthermore, the random limit measure $P$ is unique \Pas.
\end{proposition}
It is clear that \ac{as} \ac{EWC} implies its counterpart in probability with the same limit measure, and that the converse is not true in general.
\begin{remark}
    The above definition raises the question of the measurability of the map $\omega\mapsto d(\eta_n^X,P)(\omega)$.
    It is indeed measurable, since $(\eta_n^X,P)$ is measurable (cf. Theorem 14.8 in \cite{Kle2013}) and $d$ is trivially continuous.
\end{remark}
A meaningful question is then on the asymptotic behaviour of the empirical averages $\eta_n^Xf$, where $f$ is a test function.
Usual results on weak convergence hint that they should converge to $Pf$ as long as $f\in\Cb(\X;\R)$.
This is what the following result guarantees, with the additional technicality that the convergence is here again probabilistic.
In fact, this is an equivalent characterization of \ac{EWC} under an additional compactness assumption.
\begin{theorem}\label{thm:characterization ewc as}
    Let $\X$ be a Polish space, $X$ an $\X$-valued process, and $P$ a random measure on $\X$. Consider the following statements:
    \begin{enumerate}[label=(\roman*)]
        \item \label{stmt:ewc as:i} $X$ is \ac{EWC} \ac{as} with limit measure $P$;
        \item \label{stmt:ewc as:ii} $\P\left[\forall f\in\Cb(\X;\R),\limsup_{n}\lvert \eta_n^X f - Pf\rvert = 0\right]=1$;
        \item \label{stmt:ewc as:iii} for all $f\in\Cb(\X;\R)$, $\lim_{n\to\infty} \eta_n^Xf = Pf$, where the convergence is \ac{as};
    \end{enumerate}
    Then, \ref{stmt:ewc as:i} $\iff$ \ref{stmt:ewc as:ii} $\implies$ \ref{stmt:ewc as:iii}.
    If, additionally, $\X$ is compact, the implication \ref{stmt:ewc as:iii} $\implies$ \ref{stmt:ewc as:ii} also holds and all of the statements are equivalent.
\end{theorem}
\begin{theorem}\label{thm:characterization ewc prob}
    Let $\X$ be a Polish space, $X$ an $\X$-valued process, and $P$ a random measure. Consider the following statements:
    \begin{enumerate}[label=(\roman*)]
        \item \label{stmt:ewc prob:i} $X$ is \ac{EWC} in probability with limit measure $P$;
        \item \label{stmt:ewc prob:ii} for any strictly increasing sequence $(k_n)_{n\in\Nstar}\subset\Nstar$, there exists a subsequence $(k_{m_n})_{n\in\Nstar}$ such that \begin{equation*}
            \P[\forall f\in\Cb(\X;\R),\limsup_{n}\lvert \eta_{k_{m_n}}^X f - Pf\rvert = 0] = 1.
        \end{equation*}
        \item \label{stmt:ewc prob:iii} for all $f\in\Cb(\X;\R)$, $\lim_{n\to\infty} \eta_n^Xf = Pf$, where the convergence is in probability.
    \end{enumerate}
    Then, \ref{stmt:ewc prob:i} $\iff$ \ref{stmt:ewc prob:ii} $\implies$ \ref{stmt:ewc prob:iii}.
    If, additionally, $\X$ is compact, the implication \ref{stmt:ewc prob:iii} $\implies$ \ref{stmt:ewc prob:ii} also holds and all of the statements are equivalent.
\end{theorem}
We point that, in the second statement of the above two theorems, the set of which the probability is evaluated is indeed measurable.
The next results are very useful properties of \ac{EWC} processes: empirical averages of Hilbert-space-valued continuous and bounded maps also converge.
They are key in the proof of Theorem~\ref{thm:consistency svm under ewc}, but their generality makes them of independent interest.
\begin{theorem}\label{thm:averages of Cb(H)}
    Let $\X$ be a locally compact Polish space and $\H$ be a separable Hilbert space.
    Let $X$ be an $\X$-valued process, and assume that $X$ is \ac{EWC} in probability (resp. \ac{as}) with limit measure $P$.
    Then, for all $\phi\in\Cb(\X;\H)$, we have \begin{equation*}
        \lim_{n\to\infty} \eta_n^X \phi = P\phi,
    \end{equation*}
    where the convergence is in probability (resp. \ac{as}).
\end{theorem}
\begin{corollary}\label{clry:uniform averages of Cb(H)}
    Let $\X$ be a locally compact Polish space, $\H$ a separable Hilbert space, and $\F$ a separable subset of $\Cb(\X;\H)$.
    Let $X$ be an $\X$-valued process, and assume that $X$ is \ac{EWC} in probability (resp. \ac{as}) with limit measure $P$.
    Then, for any strictly increasing sequence $(k_n)_{n\in\Nstar}\subset\Nstar$, there exists a subsequence $(k_{m_n})_{n\in\Nstar}$ such that  \begin{equation*}
        \P[\forall f\in\F,\,\limsup_n \norm{\eta_{k_{m_n}}f - Pf}_\H = 0] = 1.
    \end{equation*}
    If, additionally, $X$ is \ac{EWC} \ac{as}, then \begin{equation*}
        \P[\forall f\in\F,\,\limsup_n \norm{\eta_nf - Pf}_\H = 0] = 1.
    \end{equation*}
\end{corollary}

\subsection{Examples and connections} \label{sec:connections}
\review{%
We begin with examples of \ac{EWC} processes to build intuition and establish its relevance as a generalization of existing notions before exhibiting some general connections.
If unspecified, $\X$ is a Polish space in this section.
\par
The following example is a compelling and simple instance where \ac{EWC} strictly generalizes mixing or the \ac{LLNE}.
\begin{example}\label{ex:simple EWC}
    Let $Y$ be a Rademacher variable and $X_n = (1 - \frac1n)Y$ for all $n\in\Nstar$.
    The process $X$ is \ac{EWC} \ac{as} with limit measure $\delta_{Y}$.
\end{example}
This example abstracts a dynamical system that takes an irreversible step.
Notice indeed that an equivalent definition is $X_1=Y$ and $X_{n+1} = \frac{n^2}{(n-1)(n+1)}X_n$; the irreversible step is taken at time $n=1$.
It can for instance model the position of a ball on a hill at the origin and subject to gravity: the ball falls on a random side of the hill and goes to the corresponding asymptotic position $\pm 1$.
The following example generalizes the idea so that every step is irreversible.
\begin{example}\label{ex:ewc to cantor set}
    Let $Y=(Y_n)_{n\in\Nstar}$ be \ac{iid} Rademacher random variables. Let $X_1 = \frac12+\frac13Y_1$, and $X_{n+1} = X_{n} + \frac{1}{3^{n+1}}Y_{n+1}$.
    The process $X$ converges \ac{as} to a variable $Z$ taking values in the Cantor set. 
    In particular, $X$ is \ac{EWC} with limit measure $\delta_Z$.
\end{example}
Neither of these simple examples is captured by existing assumptions, and yet they are sufficiently well-behaved --- for instance, every trajectory converges --- so that one can expect meaningful learning.
Consider, for instance, the case where data is of the form $(X_n, f(X_n)+\epsilon_n)$, with $X_n$ defined in Example~\ref{ex:ewc to cantor set}, $f$ a function one wants to learn, and $\epsilon_n$ \ac{iid} noise.
Results that guarantee successful learning with such data (under technical assumptions) are in Section~\ref{sec:consistency SVM}.
Before that, we connect \ac{EWC} to other notions, confirming the generality of the assumption.
}
\paragraph{Independent, ergodic, and mixing processes}
Consistently with our initial motivation, \ac{iid} and mixing processes are \ac{EWC} when $\X$ is compact.
More generally, any ergodic process is.
Indeed, recall that an $\X$-valued process $X$ is ergodic if it satisfies the condition of Birkhoff's pointwise ergodic theorem, that is, if there exists a measure $P\in\Pcal(\X)$ such that for every $f\in\L_1(\X,P;\R),\,\lim_{n\to\infty}\eta_n^Xf \to Pf$, where convergence is \ac{as}~
Importantly, ergodicity is the weakest of all of the above notions: $\alpha$--mixing is the weakest notion of mixing among $\beta$-- and $\phi$--mixing, and is stronger than ergodicity (see \cite{Bra2005} for details).
But then, it is immediately clear from \ref{stmt:ewc as:iii} in Theorem~\ref{thm:characterization ewc as} that such a process is \ac{EWC} \ac{as} with constant limit measure $P$.
The question is more delicate if $\X$ is not compact, as the null set of non-convergence of $\eta_n^Xf$ to $Pf$ in the definition of an ergodic process depends on the function $f$ in general.
Albeit interesting, this question is out of scope.

\paragraph{Measure-preserving dynamical systems}
Birkhoff's pointwise ergodic theorem applied to a general non-ergodic, measure-preserving dynamical systems yields convergence of $\eta_n^Xf$ to a conditional expectation, i.e., an expectation w.r.t. a random measure \cite{AQ2020}.
Hence, such trajectories are \ac{EWC} \ac{as} if $\X$ is compact.

\paragraph{Weak Convergence}
The notions of \ac{EWC} and weakly converging processes differ, despite similar naming.
The latter condition involves weak convergence of the sequence of \emph{marginals} $M_n\in\Pcal(\X)$ of the process $X$ to a non-random probability measure $M\in\Pcal(\X)$.
Intuitively, it thus involves ensemble averages (and thus disregards correlations), whereas \ac{EWC} involves time averages (and thus considers correlations).
Both concepts are independent; there exist \ac{EWC} processes that do not converge weakly, and vice-versa:
\begin{lemma}\label{lemma:non WC but EWC}
    The process $X$ defined on $\X:=\{-1,1\}$ by $X_n = (-1)^n$ \ac{as} is not weakly converging, but it is \ac{EWC} \ac{as} with limit measure $\frac12(\delta_{-1} + \delta_1)$.
\end{lemma}
\begin{proof}
    For $f\in\Cb(\X;\R)$ and $n\in\Nstar$, $M_nf = f((-1)^n)$, which does not not converge unless $f$ is constant.
    In contrast, $\eta_n^Xf$ is the Cesàro average of that sequence, which converges to $\frac12(f(-1) + f(1)) = \left[\frac12(\delta_{-1} + \delta_{1})\right]f$.
\end{proof}
\begin{lemma}\label{lemma:WC but non EWC}
    The process $X$ defined on $\X=\{0,1\}$ by $X_1\sim\B(0.5)$ and \begin{equation*}
        X_n = \begin{cases}
            X_1,&\text{if}~\floor{\log_{10}(n)}\equiv 0 \pmod{2},\\
            1-X_1,&\text{otherwise},
        \end{cases}
    \end{equation*}
    is weakly converging with limit measure $\frac12(\delta_{0} + \delta_1)$, but it is not \ac{EWC} in probability.
\end{lemma}

\paragraph{Asymptotic mean stationarity}
As announced, processes that satisfy the \ac{LLNE} are \ac{EWC} with limit measure the stationary mean.
\begin{lemma}\label{lemma:llne implies EWC}
Let $\X$ be a Polish space and $X$ be an $\X$-valued process that satisfies the \ac{WLLNE} with stationary mean $P\in\Pcal(\X)$. Then, $X$ is \ac{EWC} in probability with limit measure constant equal to $P$.
If, additionally, $X$ satisfies the \ac{SLLNE}, then $X$ is \ac{EWC} \ac{as}
\end{lemma}
This is no longer true if $X$ is only \ac{AMS}\footnote{See Definition 2.2 in \cite{SHS2009} for the definition.}, however.
For instance, the process of Lemma~\ref{lemma:WC but non EWC} is \ac{AMS} but not \ac{EWC} in probability as the empirical averages only converge \emph{on expectation}.
In fact, there is no relationship between \ac{AMS} and \ac{EWC}, as the requirements for \ac{AMS} are stronger on some aspects (\ac{AMS} requires convergence in total variation instead of weak convergence) but weaker on others (\ac{EWC} requires the whole random limit measure to exist).
This leads us to introduce the following two notions, which strengthen \ac{EWC} (resp. weaken \ac{AMS}) to relate to \ac{AMS} (resp. to \ac{EWC}), as summarized on Figure~\ref{fig:summary notions}.
\begin{definition}\label{def:esc}
    Let $\X$ be a Polish space and $\dtv$ a metric on $\Pcal(\X)$ that metrizes convergence in total variation.
    An $\X$-valued process $X$ is \emph{\acf{ESC}} in probability (resp. \ac{as}) if there exists a random measure $P$ on $\X$ such that $\dtv(\eta_n^X,P)$ is measurable, $n\in\Nstar$, and \eqref{eq:ewc} holds with $\dtv$ instead of $d$, where the convergence is in probability (resp. \ac{as}).
    Then, $P$ is unique \Pas~and is called the limit measure of $X$.
\end{definition}
\begin{definition}
    Let $\X$ be a Polish space.
    We say that an $\X$-valued process $X$ is \emph{weakly \ac{AMS}} if there exists a probability measure $P\in\Pcal(\X)$ such that \begin{equation*}
        \forall f\in\Cb(\X;\R), \lim_{n\to\infty}\frac1n\sum_{i=1}^n \E[f(X_i)] = Pf.
    \end{equation*}
    It is called the weak asymptotic mean of $X$ and is unique.
\end{definition}
\begin{remark}
    Assuming that $\dtv(\eta_n^X,P)$ is measurable in Definition~\ref{def:esc} certainly lacks elegance, but it is also solves a nontrivial technical difficulty.
    Indeed, the topology of convergence in total variation on $\Pcal(\X)$ is strictly finer than that of weak convergence in general, and the definition of random measures only gives measurability w.r.t. the Borel $\sigma$-algebra generated by the latter.
    Further exploration of this assumption is left for future work; we refer the interested reader to Chapter 4 in \cite{Kal2017}.
\end{remark}
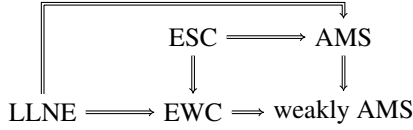
\begin{figure}
    \centering
    \begin{tikzpicture}[>=stealth,>=implies]
    \node(EWC) at (-2,0){\ac{EWC}};
    \node(ESC) at (-2,1){\ac{ESC}};
    \node(LLNE) at (-4,0){\ac{LLNE}};

    \node(wAMS) at (0,0){weakly \ac{AMS}};
    \node(AMS) at (0,1){\ac{AMS}};
    
    \draw[->,double] (LLNE) -- (EWC);
    \draw[->,double] (ESC) -- (EWC);
    \draw[->,double] (ESC) -- (AMS);
    \draw[->,double] (EWC) -- (wAMS);
    \draw[->,double] (AMS) -- (wAMS);
    \draw[->,double] (LLNE.north) -- ++(0,1.2)  -| (AMS.north);
\end{tikzpicture}
    \caption{Relation between the the different notions.
    All implications hold with the same limit measure or with the intensity measure thereof, and both in probability and \ac{as} when applicable.
    \review{%
    The implication \ac{LLNE} $\implies$ AMS is shown in \citet{SHS2009}.
    }
    }
    \label{fig:summary notions}
\end{figure}
\begin{proposition}\label{prop:summary notions}
    Let $\X$ be a Polish space, $X$ an $\X$-valued process, and $P$ a random measure on $\X$.
    If $X$ is \ac{ESC} in probability (resp. \ac{as}) with limit measure $P$, then it is both \ac{EWC} in probability (resp. \ac{as}) with the same limit measure and \ac{AMS}.
    If $X$ is \ac{EWC} in probability with limit measure $P$, then it is weakly \ac{AMS}.
    In either case, the corresponding asymptotic mean $\bar P$ satisfies $\E P = \bar P$, where $\E P$ is the usual intensity measure of a random measure.
\end{proposition}
In other words, \ac{ESC} and \ac{EWC} require the whole limit measures to exist whereas \ac{AMS} and weak \ac{AMS} only require convergence on expectation.
It is still unclear whether the implication \ac{LLNE} $\implies$ \ac{ESC} holds, as it would require a result similar to Theorems~\ref{thm:characterization ewc as} and~\ref{thm:characterization ewc prob} for \ac{ESC} processes, for which the implication (iii) $\implies$ (ii) is more challenging because of the non-separability of $\L_\infty(\X;\R)$.

\subsection{Joint processes} \label{sec:stability}
We now focus on characterizing \ac{EWC} for joint processes.
Indeed, it is a core assumption for learning, as the main assumption of Theorem~\ref{thm:consistency svm under ewc} is \ac{EWC} of the input-output process.
We begin with continuous transformations of a process, as they include among others marginalization.
\begin{theorem}\label{thm:continuous transformation}
Let $\X$ and $\Y$ be Polish spaces, $X$ be an $\X$-valued process, and $g:\X\to\Y$ be continuous.
Assume that $X$ is \ac{EWC} in probability (resp. \ac{as}) with limit measure $P$.
Then, the $\Y$-valued process $g(X)$ is \ac{EWC} in probability (resp. \ac{as}) with limit measure $Q = P\circ g^{-1}$. 
\end{theorem}
\begin{proof}
    This follows immediately by showing that \ref{stmt:ewc as:ii} in Theorem~\ref{thm:characterization ewc as} (resp. Theorem~\ref{thm:characterization ewc prob}) holds for $Y$ and $Q$.
\end{proof}
Next, we formalize the mapping of a process through a Markov kernel.
\begin{definition} \label{def:transition pairs}
    Let $\X$ and $\Y$ be Polish spaces and $p$ be a Markov kernel from $\X$ to $\Y$.
    A \emph{data set of transition pairs of $p$} is an $\X\times\Y$-valued process $Z = (X, Y)$ such that the following holds \Pas:\begin{equation}
        \forall n\in\Nstar,\,\P\left[Y_n\in \cdot \mid X_{1:n}, Y_{1:n-1}\right] = p(\cdot, X_n).
    \end{equation}
\end{definition}
An example of a data set of transition pairs is the process $((X_n,X_{n+1}))_{n\in\Nstar}$, where $X$ is a Markov chain.
\begin{theorem} \label{thm:ewc transition pairs}
    Let $\X$ and $\Y$ be compact Polish spaces and $Z=(X,Y)$ be a data set of transition pairs of a Feller-continuous Markov kernel $p$ from $\X$ to $\Y$.
    Then $Z$ is \ac{EWC} in probability (resp \ac{as}) if, and only if, $X$ is.
    In this case, with $P$ the limit measure of $X$ and $J$ that of $Z$, we have \Pas~for all $A\in\B(\X)$ and $B\in\B(\Y)$\begin{equation}
        \label{eq:joint limit measure}
        J(A\times B) = \int_A p(B, x)\d P(x).
    \end{equation}
\end{theorem}
\section{Consistency of kernel methods}\label{sec:consistency}
We now state our results on the consistency of kernel methods.
We focus first on \acp{KME} and show \emph{statistical consistency} of the standard estimator, i.e., that it recovers the \ac{KME} of the random limit measure in the infinite-sample limit.
We then attend to \acp{SVM} and show \emph{$L$-consistency}, which focuses on the achieved risk instead of on the estimator itself and is the standard notion for a learning method (see \cite{SC2008}).
\subsection{Statistical consistency of kernel mean embeddings}\label{sec:consistency kme}
\begin{theorem}\label{thm:consistency kme}
    Let $\X$ be a Polish space, $X$ an $\X$-valued process, and $\H$ a scalar-valued \ac{RKHS} of measurable functions on $\X$ with bounded kernel $k$.
    Finally, let $P$ be a random measure on $\X$.
    If $X$ is \ac{EWC} with limit measure $P$, then \begin{equation}\label{eq:convergence kme}
        \lim_{n\to\infty} \frac{1}{n}\sum_{i=1}^n k(\cdot, X_i) = \int_\X k(\cdot, x)\d P(x) =: \mu_P,
    \end{equation}
    where the convergence is in probability (resp. \ac{as}).
    Conversely, if the \ac{MMD} \begin{equation*}
        \mmd:(Q,R)\in\Pcal(\X)^2\mapsto \norm{\mu_Q-\mu_R}_\H
    \end{equation*}
    metrizes weak convergence in $\Pcal(\X)$, then \eqref{eq:convergence kme} implies that $X$ is \ac{EWC} in probability (resp. \ac{as}).
\end{theorem}
While this result seems to follow from Theorem\,\ref{thm:averages of Cb(H)}, we emphasize that it has weaker assumptions.
In particular, Theorem~\ref{thm:consistency kme} assumes neither local compactness nor a continuous kernel.
This remarkable generalization relies on the fact that the topology generated by the \ac{MMD} is coarser than that of weak convergence under the assumptions of the theorem (by Theorem 21 in \cite{SGF2010}).
The metrization of weak convergence by the \ac{MMD} is well-studied; see \cite{SBSM2023} for a comprehensive overview.
For instance, a sufficient condition is that $\X$ is compact and $k$ is bounded, continuous, and characteristic (Theorem 7 in the reference; recall that any Polish space is Hausdorff).

\subsection{$L$-consistency of SVMs}
\label{sec:consistency SVM}
We generalize consistency of a learning method to \ac{EWC} data.
\begin{definition}\label{def:consistency}
    Let $\X$ and $\Y$ be Polish spaces, $L$ be a loss function, and $Z = (X,Y)$ be an $\X\times\Y$-valued process. 
    Assume that $Z$ is \ac{EWC} with limit measure $J$.
    We say that a learning method $\Lfrak$ is $L$-consistent if, \begin{equation*}
        \lim_{n\to\infty} \Rcal_{L,J}(f_{Z_{1:n}}) = \Rcal_{L,J}^\star,
    \end{equation*}
    where convergence is in probability.
    Further, if the convergence is \ac{as}, we say that $\Lfrak$ is strongly $L$-consistent.
\end{definition}
This definition raises the immediate concern of the measurability of the sets $\{\Rcal_{L,J}(f_{Z_{1:n}}) \leq \Rcal^\star_{L,J} + \epsilon\}$, where $\epsilon>0$.
For \iac{SVM} on a Polish space with outputs in a separable Hilbert space and in a separable \ac{RKHS} of measurable functions, such sets are measurable.
This result was already known in the case of a non-random measure $J$ and scalar outputs (Lemmas 6.3 and 6.23 in \cite{SC2008}); we extend it to our setting in Appendix~\ref{apdx:measurability SVMs}.

The next theorem is our main result.
It guarantees that \acp{SVM} are consistent with \ac{EWC} data, up to technical assumptions on the loss, \ac{RKHS}, and $\Cb(\X\times\Y;\G)$.
\begin{theorem}\label{thm:consistency svm under ewc}
    Let $\X$ be a locally compact Polish space, $\G$ a separable Hilbert space, $\Y$ a complete subset of $\G$, and $\H$ a separable $\G$-valued \ac{RKHS} of continuous functions on $\X$ with bounded kernel $K$ and continuous feature map $\Phi:x\mapsto K(\cdot,x)$.
    Assume that $\Cb(\X\times\Y;\H)$ is separable.
    Let $L$ be a convex, continuously differentiable, locally Lipschitz continuous, and locally bounded loss function.
    Let $Z = (X,Y)$ be an $\X\times\Y$-valued process, and assume that $Z$ is \ac{EWC} in probability (resp. \ac{as}) with limit measure $J$.
    Assume that $\H$ is $(L,J)$-rich, \Pas~
    Then, there exists a sequence of strictly positive real numbers $(\lambda_n)_{n\in\Nstar}$ such that $\lim_{n\to\infty}\lambda_n = 0$ and the $(\lambda_n)$-\ac{SVM} is $L$-consistent (resp. strongly $L$-consistent) for $Z$.
\end{theorem}
We emphasize again that usual consistency results do \emph{not} assume that the loss function is continuously differentiable.
The necessity for this requirement comes directly from the weak convergence that \ac{EWC} guarantees: this assumption is more general than what is usually assumed, but comes at the price of weaker guarantees.
\review{
Furthermore, Theorem~\ref{thm:consistency svm under ewc} does not indicate at what speed the sequence $(\lambda_n)_{n\in\Nstar}$ should decrease.
Unfortunately, no such speed exists without further assumptions, for two reasons.
First, similarly to works with assumptions stronger than \ac{EWC}, there is no such choice of $(\lambda_n)_{n\in\Nstar}$ uniform in the limit measure $J$; see for instance\,\citet{SHS2009} or Section 6.1 in \citet{SC2008}.
Second, \ac{EWC} allows the data to approximate the limit distribution arbitrarily slowly, adding another layer of necessary assumptions.
In other words, a finite-sample analysis requires handling both the classical question of bounding the approximation error function\,\cite{SC2008} and imposing a minimum speed on the \ac{EWC} convergence.
}

\section{Conclusion and outlook}

We introduce the new notion of empirical weak convergence to address the consistency of kernel methods in the presence of dependent data.
We show consistency for \acp{SVM} for \ac{EWC} processes and allow an infinite-dimensional output, which is crucial to include related methods such as \acp{CKME}.
We discuss in details how \ac{EWC} relates to and generalizes the existing usual assumptions for consistency, establishing it as a suitable basis paving the way for a more general theory of learning with dependent data.
\par
An important open question is on the generalization of Theorem~\ref{thm:consistency svm under ewc} to non-separable sets of continuous bounded functions.
Indeed, that assumption is typically achieved through compactness of the input and output sets, which is relatively restrictive.
Another relevant topic is to extend the theorem to stronger forms of convergence to a random measure, such as \ac{ESC}.
Finally, the topic of learning rates is a promising area for future work.
While it is clear that the convergence involved in \ac{EWC} may be arbitrarily slow, translating bounds on this speed into oracle inequalities for \acp{SVM} is an interesting extension of our results.

\section*{Acknowledgements}
We thank Christian Fiedler and the anonymous reviewers for detailed and helpful comments.

\section*{Impact statement}
This paper presents work whose goal is to advance the field of Machine Learning. 
There are many potential societal consequences of our work, none which we feel must be specifically highlighted here.

\bibliography{references}

\begin{thebibliography}{40}
\providecommand{\natexlab}[1]{#1}
\providecommand{\url}[1]{\texttt{#1}}
\expandafter\ifx\csname urlstyle\endcsname\relax
  \providecommand{\doi}[1]{doi: #1}\else
  \providecommand{\doi}{doi: \begingroup \urlstyle{rm}\Url}\fi

\bibitem[Adams \& Quas(2020)Adams and Quas]{AQ2020}
Adams, T. and Quas, A.
\newblock \emph{Ergodicity and Mixing Properties}, pp.\  1--26.
\newblock Springer, 2020.

\bibitem[Aliprantis \& Border(2006)Aliprantis and Border]{AB2006}
Aliprantis, C.~D. and Border, K.~C.
\newblock \emph{Infinite dimensional analysis}.
\newblock Springer, 3 edition, 2006.

\bibitem[Bauer(2011)]{Bau2011}
Bauer, H.
\newblock \emph{Measure and integration theory}, volume~26.
\newblock Walter de Gruyter, 2011.

\bibitem[Beaglehole et~al.(2023)Beaglehole, Belkin, and Pandit]{BBP2023}
Beaglehole, D., Belkin, M., and Pandit, P.
\newblock On the inconsistency of kernel ridgeless regression in fixed dimensions.
\newblock \emph{SIAM Journal on Mathematics of Data Science}, 5\penalty0 (4):\penalty0 854--872, 2023.

\bibitem[Beer(1993)]{Bee1993}
Beer, G.
\newblock \emph{Topologies on closed and closed convex sets}, volume 268.
\newblock Springer Science \& Business Media, 1993.

\bibitem[Bertsekas \& Shreve(1996)Bertsekas and Shreve]{BS1996}
Bertsekas, D. and Shreve, S.~E.
\newblock \emph{Stochastic optimal control: the discrete-time case}, volume~5.
\newblock Athena Scientific, 1996.

\bibitem[Bogachev(2007)]{BR2007}
Bogachev, V.~I.
\newblock \emph{Measure theory}, volume~2.
\newblock Springer, 2007.

\bibitem[Bradley(2005)]{Bra2005}
Bradley, R.~C.
\newblock {Basic Properties of Strong Mixing Conditions. A Survey and Some Open Questions}.
\newblock \emph{Probability Surveys}, 2:\penalty0 107 -- 144, 2005.

\bibitem[Buchholz(2022)]{Buc2022}
Buchholz, S.
\newblock Kernel interpolation in sobolev spaces is not consistent in low dimensions.
\newblock In \emph{Conference on Learning Theory}, pp.\  3410--3440. PMLR, 2022.

\bibitem[Buisson-Fenet et~al.(2020)Buisson-Fenet, Solowjow, and Trimpe]{BST2020}
Buisson-Fenet, M., Solowjow, F., and Trimpe, S.
\newblock Actively learning gaussian process dynamics.
\newblock In \emph{Learning for dynamics and control}, pp.\  5--15. PMLR, 2020.

\bibitem[Caponnetto \& De~Vito(2007)Caponnetto and De~Vito]{CDV2007}
Caponnetto, A. and De~Vito, E.
\newblock Optimal rates for the regularized least-squares algorithm.
\newblock \emph{Foundations of Computational Mathematics}, 7:\penalty0 331--368, 2007.

\bibitem[Carmeli et~al.(2006)Carmeli, De~Vito, and Toigo]{CDVT2006}
Carmeli, C., De~Vito, E., and Toigo, A.
\newblock Vector valued reproducing kernel hilbert spaces of integrable functions and mercer theorem.
\newblock \emph{Analysis and Applications}, 4\penalty0 (04):\penalty0 377--408, 2006.

\bibitem[Castaing \& Valadier(1977)Castaing and Valadier]{CV1977}
Castaing, C. and Valadier, M.
\newblock \emph{Convex Analysis and Measurable Multifunctions}.
\newblock Springer, 1977.

\bibitem[Cs{\"o}rg{\"o}(1968)]{Csoe1968}
Cs{\"o}rg{\"o}, M.
\newblock On the strong law of large numbers and the central limit theorem for martingales.
\newblock \emph{Transactions of the American Mathematical Society}, 131\penalty0 (1):\penalty0 259--275, 1968.

\bibitem[Grünewälder et~al.(2012)Grünewälder, Lever, Baldassarre, Patterson, Gretton, and Pontil]{GLB2012}
Grünewälder, S., Lever, G., Baldassarre, L., Patterson, S., Gretton, A., and Pontil, M.
\newblock Conditional mean embeddings as regressors.
\newblock In \emph{International Conference on Machine Learning}, pp.\  1803--1810, 2012.

\bibitem[Hanneke(2021)]{Han2021}
Hanneke, S.
\newblock Learning whenever learning is possible: Universal learning under general stochastic processes.
\newblock \emph{Journal of Machine Learning Research}, 22:\penalty0 5751--5866, 2021.

\bibitem[Henrikson(1999)]{Hen1999}
Henrikson, J.
\newblock {Completeness and total boundedness of the Hausdorff metric}.
\newblock \emph{MIT Undergraduate Journal of Mathematics}, 1\penalty0 (69-80):\penalty0 10, 1999.

\bibitem[Hyt{\"o}nen et~al.(2016)Hyt{\"o}nen, Van~Neerven, Veraar, and Weis]{HVNVW2016}
Hyt{\"o}nen, T., Van~Neerven, J., Veraar, M., and Weis, L.
\newblock \emph{Analysis in Banach spaces}, volume~1.
\newblock Springer, 2016.

\bibitem[Irle(1997)]{Irl1997}
Irle, A.
\newblock On consistency in nonparametric estimation under mixing conditions.
\newblock \emph{Journal of multivariate analysis}, 60\penalty0 (1):\penalty0 123--147, 1997.

\bibitem[Kallenberg(2017)]{Kal2017}
Kallenberg, O.
\newblock \emph{Random measures, theory and applications}.
\newblock Springer, 2017.

\bibitem[Klebanov et~al.(2020)Klebanov, Schuster, and Sullivan]{KSS2020}
Klebanov, I., Schuster, I., and Sullivan, T.~J.
\newblock A rigorous theory of conditional mean embeddings.
\newblock \emph{SIAM Journal on Mathematics of Data Science}, pp.\  583--606, 2020.

\bibitem[Klenke(2013)]{Kle2013}
Klenke, A.
\newblock \emph{Probability theory: a comprehensive course}.
\newblock Springer Science \& Business Media, 2013.

\bibitem[Kruger(2009)]{Kruger2009}
Kruger, A.~Y.
\newblock \emph{Nonsmooth analysis: Fr{\'e}chet subdifferentialsNonsmooth Analysis: Fr{\'e}chet Subdifferentials}, pp.\  2651--2658.
\newblock Springer, 2009.

\bibitem[Li \& Li(2009)Li and Li]{li2009criteria}
Li, Y. and Li, J.
\newblock {Criteria for Feller transition functions}.
\newblock \emph{Journal of mathematical analysis and applications}, 359\penalty0 (2):\penalty0 653--665, 2009.

\bibitem[Li et~al.(2022)Li, Meunier, Mollenhauer, and Gretton]{LMMG2022}
Li, Z., Meunier, D., Mollenhauer, M., and Gretton, A.
\newblock Optimal {Rates} for {Regularized} {Conditional} {Mean} {Embedding} {Learning}.
\newblock \emph{Advances in Neural Information Processing Systems}, 2022.

\bibitem[Megginson(2012)]{Meg2012}
Megginson, R.~E.
\newblock \emph{An introduction to Banach space theory}.
\newblock Springer, 2012.

\bibitem[Mnih et~al.(2013)Mnih, Kavukcuoglu, Silver, Graves, Antonoglou, Wierstra, and Riedmiller]{MKS2013}
Mnih, V., Kavukcuoglu, K., Silver, D., Graves, A., Antonoglou, I., Wierstra, D., and Riedmiller, M.
\newblock Playing atari with deep reinforcement learning.
\newblock \emph{arXiv preprint arXiv:1312.5602}, 2013.

\bibitem[Mollenhauer \& Koltai(2020)Mollenhauer and Koltai]{MK2020}
Mollenhauer, M. and Koltai, P.
\newblock Nonparametric approximation of conditional expectation operators.
\newblock \emph{arXiv preprint arXiv:2012.12917}, 2020.

\bibitem[Park \& Muandet(2022)Park and Muandet]{PM2022}
Park, J. and Muandet, K.
\newblock Regularised {Least}-{Squares} {Regression} with {Infinite}-{Dimensional} {Output} {Space}.
\newblock \emph{arXiv preprint arXiv:2010.10973}, 2022.

\bibitem[Rakhlin \& Zhai(2019)Rakhlin and Zhai]{RZ2019}
Rakhlin, A. and Zhai, X.
\newblock Consistency of interpolation with laplace kernels is a high-dimensional phenomenon.
\newblock In \emph{Conference on Learning Theory}, pp.\  2595--2623. PMLR, 2019.

\bibitem[Sch{\"o}lkopf \& Smola(1998)Sch{\"o}lkopf and Smola]{SS1998}
Sch{\"o}lkopf, B. and Smola, A.~J.
\newblock \emph{Learning with kernels}, volume~4.
\newblock Citeseer, 1998.

\bibitem[Simchowitz et~al.(2018)Simchowitz, Mania, Tu, Jordan, and Recht]{simchowitz2018learning}
Simchowitz, M., Mania, H., Tu, S., Jordan, M.~I., and Recht, B.
\newblock Learning without mixing: Towards a sharp analysis of linear system identification.
\newblock In \emph{Conference On Learning Theory}, pp.\  439--473. PMLR, 2018.

\bibitem[Simon-Gabriel et~al.(2023)Simon-Gabriel, Barp, Sch{\"o}lkopf, and Mackey]{SBSM2023}
Simon-Gabriel, C.-J., Barp, A., Sch{\"o}lkopf, B., and Mackey, L.
\newblock Metrizing weak convergence with maximum mean discrepancies.
\newblock \emph{Journal of Machine Learning Research}, 24\penalty0 (184):\penalty0 1--20, 2023.

\bibitem[Sriperumbudur et~al.(2010)Sriperumbudur, Gretton, Fukumizu, Sch{\"o}lkopf, and Lanckriet]{SGF2010}
Sriperumbudur, B.~K., Gretton, A., Fukumizu, K., Sch{\"o}lkopf, B., and Lanckriet, G.~R.
\newblock Hilbert space embeddings and metrics on probability measures.
\newblock \emph{Journal of Machine Learning Research}, 11:\penalty0 1517--1561, 2010.

\bibitem[Steinwart \& Christmann(2008)Steinwart and Christmann]{SC2008}
Steinwart, I. and Christmann, A.
\newblock \emph{Support vector machines}.
\newblock Springer, 2008.

\bibitem[Steinwart \& Christmann(2009)Steinwart and Christmann]{SC2009}
Steinwart, I. and Christmann, A.
\newblock Fast learning from non-iid observations.
\newblock \emph{Advances in neural information processing systems}, 22, 2009.

\bibitem[Steinwart et~al.(2009)Steinwart, Hush, and Scovel]{SHS2009}
Steinwart, I., Hush, D., and Scovel, C.
\newblock Learning from dependent observations.
\newblock \emph{Journal of Multivariate Analysis}, 100:\penalty0 175--194, 2009.

\bibitem[Von~Luxburg \& Sch{\"o}lkopf(2011)Von~Luxburg and Sch{\"o}lkopf]{VLS2011}
Von~Luxburg, U. and Sch{\"o}lkopf, B.
\newblock Statistical learning theory: Models, concepts, and results.
\newblock In \emph{Handbook of the History of Logic}, volume~10, pp.\  651--706. Elsevier, 2011.

\bibitem[von Rohr et~al.(2021)von Rohr, Neumann-Brosig, and Trimpe]{vRNT2021}
von Rohr, A., Neumann-Brosig, M., and Trimpe, S.
\newblock Probabilistic robust linear quadratic regulators with gaussian processes.
\newblock In \emph{Learning for Dynamics and Control}, pp.\  324--335. PMLR, 2021.

\bibitem[Ziemann \& Tu(2022)Ziemann and Tu]{ZT2022}
Ziemann, I. and Tu, S.
\newblock Learning with little mixing.
\newblock \emph{Advances in Neural Information Processing Systems}, 35:\penalty0 4626--4637, 2022.

\end{thebibliography}
\bibliographystyle{icml2024}

\newpage
\appendix
\onecolumn
\section{Proofs}\label{apdx:proofs}
\subsection{Proofs for Section \ref{sec:ewc definition}}
\begin{proof}[Proof of Proposition\,\ref{prop:uniqueness limit}]
    Let $d_1$ and $d_2$ be two metrics that metrize weak convergence on $\Pcal(\X)$.
    By definition, sequences of $\Pcal(\X)$ that converge for $d_1$ also converge for $d_2$, and vice versa.
    It immediately follows that the notion of \ac{EWC} \ac{as} is unaffected by the choice of the metric.
    Furthermore, we can apply Corollary 20.8 in \cite{Bau2011} to the sequence $(d(\eta_n^X,P))_{n\in\Nstar}$ to show that a process $X$ is \ac{EWC} in probability if, and only if, for any strictly increasing sequence $k\subset\Nstar$, there exists a subsequence $(k_{m_n})\subset\Nstar$ such that $\lim_{n\to\infty}\eta_{k_{m_n}}^X = P$, where the convergence is \ac{as}~
    Since \ac{as} convergence of random measures is independent of the chosen metric that metrizes weak convergence, we deduce that so is the notion of a process \ac{EWC} in probability.

    We are left to show uniqueness of the limit measure. 
    We only consider the case where $X$ is \ac{EWC} in probability as the other case follows immediately since any process \ac{EWC} \ac{as} is also \ac{EWC} in probability with the same limit measure.
    The result follows immediately from uniqueness of the limit in probability in a metric space; we repeat the proof for completeness.
    Let $d$ be a metric that metrizes weak convergence on $\Pcal(\X)$ and $P$ and $Q$ be two limit (random) measures of $X$.
    For any $\epsilon>0$ and $n\in\Nstar$, we have \begin{align*}
        \P[d(P,Q)>\epsilon]
            &\leq \P[d(\eta_n^X,P)>\epsilon~\text{or}~d(\eta_n^X,Q)>\epsilon]\\
            &\leq \P[d(\eta_n^X,P)>\epsilon] + \P[d(\eta_n^X,Q)>\epsilon],
    \end{align*}
    by the triangle inequality for $d$ and the union bound.
    Since the \ac{rhs} goes to $0$ as $n\to\infty$, we deduce that $d(P,Q) = 0$ \ac{as}, which is the desired result. 
\end{proof}

\begin{proof}[Proof of Theorem~\ref{thm:characterization ewc as}]
    The equivalence \ref{stmt:ewc as:i} $\iff$ \ref{stmt:ewc as:ii} immediately follows from the definition of weak convergence.
    Indeed, we have \begin{align*}
        \{\eta_n^X\cvw P\} 
            &= \{\forall f\in\Cb(\X;\R), \lim_{n\to\infty}\eta_n^X f = Pf\}\\
            &= \{\forall f\in\Cb(\X;\R), \limsup_{n}\lvert\eta_n^X f - Pf\rvert = 0\}.
    \end{align*}
    Since $d$ is continuous, it follows that both sets are measurable.
    Taking probabilities on both sides shows the equivalence.

    Next, it is also clear that \ref{stmt:ewc as:ii} $\implies$ \ref{stmt:ewc as:iii}, since for any $g\in\Cb(\X;\R)$, \begin{equation*}
        \{\forall f\in\Cb(\X;\R), \lim_{n\to\infty}\eta_n^X f = Pf\} \subset \{\lim_{n\to\infty}\eta_n^X g = Pg\},
    \end{equation*}
    and thus the \ac{rhs} has probability $1$ if \ref{stmt:ewc as:ii} holds.

    Finally, we further assume that $\X$ is compact and show \ref{stmt:ewc as:iii} $\implies$ \ref{stmt:ewc as:ii}.
    Assume that \ref{stmt:ewc as:iii} holds.
    By compactness and separability of $\X$, $(\Cb(\X;\R),\norm{\cdot}_\infty)$ is separable; let $\Phi=(\phi_m)_{m\in\Nstar}$ be a dense family thereof.
    Define for all $n\in\Nstar$ the event \begin{equation*}
        N_m = \left\{\limsup_n \lvert\eta_n^X \phi_m - P\phi_m\rvert>0\right\}.
    \end{equation*}
    It is the set of non-convergence of $\eta_n^X\phi_m$ to $P\phi_m$.
    Let \begin{equation*}
        N = \bigcup_{m=1}^\infty N_m.
    \end{equation*}
    By assumption, $\P[N_m] = 0$ for all $m\in\Nstar$ and thus $\P[N] = 0$.
    Now, let $f\in\Cb(\X;\R)$ be fixed, and take $(\psi_m)_{m\in\Nstar}\subset\Phi$ be such that $\norm{\psi_m - f}_{\infty}\to 0$ as $m\to\infty$.
    Let $\omega\in\Omega\setminus N$; from now on, we consider that all variables are evaluated in $\omega$ but drop the explicit dependency for conciseness.
    Let $\epsilon>0$ be arbitrary, and take $M\in\Nstar$ such that for all $m\geq M$, $\norm{\psi_m-f}_\infty \leq\frac\epsilon2$.
    Then, for all $m\geq M$ and $n\in\Nstar$, \begin{equation}\label{eq:bound proof characterization ewc}\begin{split}
        \left\lvert\eta_n^Xf - Pf\right\rvert 
            &\leq \left\lvert(\eta_n^X-P)(f-\psi_m)\right\rvert + \left\lvert\eta_n^X \psi_m - P\psi_m\right\rvert\\
            &\leq \lvert\eta_n^X - P\rvert\cdot\lvert f-\psi_m\rvert + \left\lvert\eta_n^X \psi_m - P\psi_m\right\rvert\\
            &\leq 2\cdot \frac\epsilon2+\left\lvert\eta_n^X \psi_m - P\psi_m\right\rvert,
    \end{split}\end{equation}
    where the second inequality comes from the triangle inequality for the signed measure $\eta_n^X - P$ of which $\lvert\eta_n^X - P\rvert$ is the absolute variation, and the third uses the fact that $\lvert \eta_n^X - P\rvert\ind_\X \leq \eta_n^X\ind_\X +P\ind_\X = 2$.
    Taking the $\limsup$ over $n$ thus yields \begin{equation*}
        \limsup_n\left\lvert\eta_n^Xf - Pf\right\rvert\leq \epsilon,
    \end{equation*}
    for all $\epsilon > 0$.
    Indeed, the second term vanishes by assumption on $\omega$.
    We deduce that $\limsup_n\left\lvert\eta_n^Xf - Pf\right\rvert = 0$ for all $f\in\Cb(\X;\R)$ and $\omega\in\Omega\setminus N$, which concludes the proof.
\end{proof}

\begin{proof}[Proof of Theorem~\ref{thm:characterization ewc prob}]
    Let $d$ be a metric that metrizes weak convergence on $\Pcal(\X)$.
    The equivalence \ref{stmt:ewc as:i} $\iff$ \ref{stmt:ewc as:ii} immediately follows from Corollary (20.8) in \cite{Bau2011}. 
    Indeed, by the cited corollary, the process $X$ is \ac{EWC} with limit measure $P$ if, and only if, for any strictly increasing sequence $k\subset\Nstar$, there exists another strictly increasing sequence $m\subset\Nstar$ such that $\lim_{n\to\infty}d(\eta_{k_{m_n}}^X,P) = 0$ with probability $1$.
    Then, for any fixed such sequence $k$ and $m$, the definition of weak convergence shows that $\lim_{n\to\infty}d(\eta_{k_{m_n}}^X,P) = 0$ holds with probability $1$ if, and only if, the event $\{\forall f\in\Cb(\X;\R), \lim_{n\to\infty}\eta_{k_{m_n}}f = Pf\}$ has probability $1$.
    The equivalence follows.

    It is also clear that \ref{stmt:ewc prob:ii} implies \ref{stmt:ewc prob:iii} by using the converse implication of Corollary (20.8) in \cite{Bau2011}.
    Indeed, \ref{stmt:ewc prob:ii} immediately implies that for all $f\in\Cb(\X;\R)$ and all strictly increasing sequence $k\subset\Nstar$, there exists another strictly increasing sequence $m\subset\Nstar$ such that $\lim_{n\to\infty}\eta_{k_{m_n}}f = Pf$ \ac{as}~
    The cited result then implies \ref{stmt:ewc prob:iii}.

    Finally, we further assume that $\X$ is compact and show \ref{stmt:ewc as:iii} $\implies$ \ref{stmt:ewc as:ii}.
    Let $k\subset\Nstar$ be a strictly increasing sequence.
    We proceed by constructing a strictly increasing sequence $m\subset\Nstar$ so that the event $\{\forall f\in\Cb(\X;\R)$, $\lim_{n\to\infty}\eta_{k_{m_n}}^Xf= Pf\}$ has probability $1$; i.e., neither the subsequence $m$ nor the null set of non-convergence depend on the chosen function $f$, contrary to the ones given by \ref{stmt:ewc prob:iii}.
    Since $\X$ is compact, $(\Cb(\X;\R),\norm{\cdot}_\infty)$ is separable; let $\Phi=(\phi_p)_{p\in\Nstar}$ be a dense family thereof.
    The construction of the sequence $m$ follows from a diagonal argument.
    First, we construct inductively a double-indexed sequence $(m_{n,p})_{(n,p)\in(\Nstar)^2}$ that satisfies the following properties for all $p\in\Nstar$:\begin{itemize}
        \item the sequence $(m_{n,p})_{n\in\Nstar}$ is strictly increasing;
        \item $\lim_{n\to\infty}\eta_{k_{m_{n,p}}}\phi_p = P\phi_p$ on all of $\Omega\setminus N_p$ for some measurable $N_p\in\A$ with $\P[N_p] = 0$;
        \item $(k_{m_{n,p+1}})_{n\in\Nstar}$ is a subsequence of $(k_{m_{n,p}})_{n\in\Nstar}$.
    \end{itemize}
    Such a construction is indeed possible by applying \ref{stmt:ewc prob:iii} and Corollary 20.8 in \cite{Bau2011} to the sequence $(\eta_{k_{m_{n,p}}}^X\phi_{p+1})_{n\in\Nstar}$ at every induction step; we forego detailing it for conciseness.
    Then, we define inductively the sequence $m$ as follows.
    We first take $m_1 := m_{1,1}$.
    Then, assuming that $m_i$ is constructed for all $i\leq p$ for some $p\in\Nstar$, define \begin{equation*}
        m_{p+1} := \min\{m_{i,p+1}\mid i\in\Nstar\land m_{i,p+1}>m_p\}.
    \end{equation*}
    The set over which we minimize is indeed nonempty since $m_{i,p+1}\to\infty$ as $i\to\infty$.
    We claim that the hereby constructed sequence $m$ has the announced property.
    First, $m$ is clearly strictly increasing, by construction.
    Next, let $N = \bigcup_{p\in\Nstar}N_p$ and take $\omega\in\Omega\setminus N$; from now on, we consider that all variables are evaluated in $\omega$ but drop the explicit dependency for conciseness.
    For any $p\in\Nstar$, the sequence $m$ satisfies $\lim_{n\to\infty}{\eta_{k_{m_n}}}\phi_p = P\phi_p$.
    Indeed, by construction, the sequence $(k_{m_n})_{n\geq p}$ is a subsequence of the sequence $(k_{m_{n,p}})_{n\geq p}$, and $\lim_{n\to\infty}\eta_{k_{m_{n,p}}}^X\phi_p = P\phi_p$.
    Now, let $f\in\Cb(\X;\R)$ be arbitrary, and take $(\psi_q)_{q\in\Nstar}\subset\Phi$ be such that $\norm{\psi_q - f}_{\infty}\to 0$ as $n\to\infty$.
    Let $\epsilon>0$, and take $Q\in\Nstar$ such that $\norm{f-\psi_q}_\infty\leq\frac\epsilon2$ for all $q\geq Q$.
    The exact same calculations as the ones in \eqref{eq:bound proof characterization ewc} with the index $k_{m_n}$ instead of $n$ directly show that for all $n\in\Nstar$ and $q\geq Q$, \begin{equation*}
        \lvert\eta_{k_{m_n}}f - Pf\rvert \leq \epsilon + \lvert\eta_{k_{m_n}}\psi_q - P\psi_q\rvert.
    \end{equation*}
    The second term vanishes when taking the $\limsup$ over $n$ by the property of the sequence $m$.
    This shows that for all $f\in\Cb(\X;\R)$ and $\epsilon>0,\limsup_n \lvert\eta_{k_{m_n}}f - Pf\rvert \leq \epsilon$, and thus the \ac{lhs} is equal to $0$.
    This is true for all $\omega\in\Omega\setminus N$ and $N$ has probability $0$; hence, $\lim_{n\to\infty}\eta_{k_{m_n}}f = Pf$, where the convergence is almost sure.
    Since the sequence $m$ is independent of $f$, this shows the result and concludes the proof.
\end{proof}

\begin{proof}[Proof of Theorem~\ref{thm:averages of Cb(H)}]
    We begin by showing the claim in the case where $X$ is \ac{EWC} \ac{as}, and assume without loss of generality that for all $x\in\X$, $\norm{\phi(x)}_\H \leq 1$.
    \par
    Let $\epsilon > 0$. 
    We apply theorem~\ref{theorem:prokhorov for random measures} to the countable family of random measures $\Q = \{\eta_n^X\mid n\in\Nstar\}\cup\{P\}$, which is sequentially compact \ac{as} by assumption (here, we rely on completeness of $(\Omega,\A)$ to guarantee the measurability of the event).
    Therefore, there exists a random compact set $K_\epsilon$ such that the following holds \ac{as}:
    \begin{equation}
        \label{eq:pointwise properties random elements}
        \begin{split}
        P(K_\epsilon) &\geq 1-\epsilon,\\
        \forall n\in\Nstar,~\eta_n^X(K_\epsilon)&\geq 1-\epsilon.
    \end{split}
    \end{equation}
    Next, since $\H$ is a Hilbert space, it has the $b$-approximation property for some $b>0$ (see the discussion after Definition~4.1.34 in \cite{Meg2012}).
    Applying this property to the compact set $\phi(K_\epsilon)$ shows that for all $\omega\in \Omega$, there exists a bounded operator $S(\omega):\H\to\H$ with finite rank $m(\omega)\in\Nstar$ such that \begin{equation}\label{eq:approximation operator}
        \forall x\in K_{\epsilon}(\omega), \norm{S(\omega)\phi(\omega)(x) - \phi(\omega)(x)}_\H\leq \epsilon,~\text{and}~\norm{S(\omega)}_{\L(\H)}\leq b.
    \end{equation}
    It is clear from the proof of the bounded approximation property (Theorem 4.1.33 in \cite{Meg2012}) that the map $\omega\mapsto S(\omega)$ can be taken to be measurable; it defines a random operator $S:\H\to\H$ with finite rank $m$ that satisfies \eqref{eq:approximation operator} \ac{as}~
    We now take an $\H$-valued process\footnote{There is indeed a choice of the map $\omega\mapsto e(\omega)$ that is measurable from separability of $\H$; we skip the proof for conciseness.} $e$ such that $(e_1,\dots,e_m)$ is \ac{as} \iac{onb} of $S\H$, and let \begin{equation}\label{eq:definition fj's}
        \forall j\in\Nstar,f_j:\omega\in\Omega \mapsto \scalar{e_j(\omega),\,S(\omega)\phi(\omega)(\cdot)}_\H\in\R^\X.
    \end{equation}
    By continuity of $\phi$, of $S$, and of the scalar product, the process $f$ takes values in $\Cb(\X;\R)$.
    \par
    We are now equipped to show the convergence announced.
    We have the following bound: \begin{equation}\label{eq:triangle inequality}
        \norm{\eta_n^X\phi - P\phi}_\H \leq \norm{\eta_n^X\phi - \eta_n^XS\phi}_\H + \norm{\eta_n^X S\phi - PS\phi}_\H + \norm{PS\phi - P\phi}_\H.
    \end{equation}
    We further bound these three terms separately to show \ac{as} convergence of the \ac{lhs} to $0$.
    For this, we leverage the equivalent characterization of \ac{as} convergence of Lemma 20.6 in \cite{Bau2011} and show that for all $\alpha>0$,\begin{equation*}
        \lim_{n\to\infty}\P\left[\sup_{\ell\geq n}\norm{\eta_\ell^X\phi - P\phi}_\H>\alpha\right] = 0.
    \end{equation*}
    To this end, we first apply that same characterization to the variable $\sqrt{m}\sup_{j\in\{1,\dots,m\}}\lvert\eta_n^X f_j - Pf_j\rvert$, which converges to $0$ \ac{as}, yielding the existence of $n_\epsilon>0$ such that the following holds with probability not less than $1-\epsilon$: 
    \begin{equation}
        \label{eq:assumptions wp 1-eps}
        \begin{split}
            \sup_{n\geq n_\epsilon}\sup_{j\in\{1,\dots,m\}}\lvert\eta_n^X f_j - Pf_j\rvert &\leq \epsilon\cdot m^{-1/2}.
        \end{split}
    \end{equation}
    Let now $A\in \A$ be an event with probability $1$ such that \begin{enumerate*}
        \item \eqref{eq:pointwise properties random elements} and \eqref{eq:approximation operator} hold; and 
        \item $(e_1,\dots,e_m)$ is \iac{onb} of $S\H$ for all $\omega\in A$, $n\geq n_\epsilon$, and $\omega\in A$ that satisfies \eqref{eq:assumptions wp 1-eps}.\end{enumerate*}
    We assume that all variables are evaluated in $\omega\in A$, but drop the explicit dependency for conciseness until mentioned otherwise.
    The last term in \eqref{eq:triangle inequality} is bounded as follows: \begin{align*}
        \norm{P S\phi - P \phi}_\H 
            &\leq P\norm{S\phi - \phi}_\H\\
            &= P\norm{S\phi - \phi}_\H\cdot\ind_{K_\epsilon} + P\norm{S\phi - \phi}_\H\cdot\ind_{\X\setminus K_\epsilon}\\
            &\leq \epsilon P\ind_{K_\epsilon} + \sup_{x\in \X\setminus K_\epsilon}\norm{S\phi(x) - \phi(x)}_\H P\ind_{\X\setminus K_\epsilon}\\
            &\leq \epsilon + \sup_{x\in \X\setminus K_\epsilon}\norm{S\phi(x) - \phi(x)}_\H\cdot \epsilon\\
            &\leq \epsilon + (b+1)\cdot \epsilon = (b+2)\epsilon,
    \end{align*}
    where we use the facts that $\norm{S\phi - \phi}_\H\cdot\ind_{K_\epsilon} \leq \epsilon\cdot \ind_{K_\epsilon}$, $P(K_\epsilon)\geq 1-\epsilon$, $\sup_x\norm{\phi(x)}_\H\leq 1$, and $\norm{S}_{\L(\H)}\leq b$.
    The first term in \eqref{eq:triangle inequality} is bounded similarly, as the same computations with $\eta_n^X$ instead of $P$ show \begin{equation*}
        \norm{\eta_n^X\phi - \eta_n^XS\phi}_\H 
            \leq \eta_n^X\norm{S\phi - \phi}_\H\cdot\ind_{K_\epsilon} + \eta_n^X\norm{S\phi - \phi}_\H\cdot\ind_{\X\setminus K_\epsilon}\leq (b+2)\epsilon.
    \end{equation*}
    where we use $\eta_n^X (K_\epsilon) \geq 1-\epsilon$.
    Finally, we bound the second term by leveraging the fact that $S$ has finite rank: \begin{align*}
        \norm{\eta_n^X S\phi - PS\phi}_\H &= \left(\sum_{j=1}^m\scalar{e_j,\,\eta_n^X S\phi - PS\phi}^2\right)^{1/2}\\
        &= \left(\sum_{j=1}^m(\eta_n^Xf_j - Pf_j)^2\right)^{1/2}\\
        &\leq \sqrt{m}\sup_{j\in\{1,\dots,m\}} \lvert\eta_n^Xf_j - Pf_j\rvert\\
        &\leq \epsilon,
    \end{align*}
    where the last inequality comes from the fact that $\omega$ satisfies \eqref{eq:assumptions wp 1-eps}.
    Since $n\geq n_\epsilon$ was arbitrary, we have shown that \begin{equation}\label{eq:bound sup 10epsilon}
        \sup_{n\geq n_\epsilon} \norm{\eta_n^X\phi - P\phi} \leq (b+2)\epsilon + \epsilon + (b+2)\epsilon = (2b+5)\epsilon.
    \end{equation}
    This bound holds for any $\omega\in A$ that satisfies \eqref{eq:assumptions wp 1-eps}; therefore, we deduce that it holds with probability at least $1-\epsilon$.
    Since $\epsilon>0$ was arbitrary, the conclusion follows as announced by leveraging (20.6) in \cite{Bau2011}.
    \par
    Now, we handle the case when $X$ is only \ac{EWC} in probability.
    Corollary 20.8 in \cite{Bau2011} shows that this is equivalent to the fact that any subsequence $(d(\eta_{k_n}^X,P))_{n\in\Nstar}$ of $(d(\eta_n^X,P))_{n\in\Nstar}$ has itself a subsequence that converges to $0$ \ac{as}, say, $(d(\eta_{k_{m_n}}^X,P))_{n\in\Nstar}$.
    Here $k:\Nstar\to\Nstar$ and $m:\Nstar\to\Nstar$ are strictly increasing sequences.
    But then, $(\eta_{k_{m_n}}^X)_{n\in\Nstar}$ converges weakly to $P$ with probability $1$.
    We can thus repeat the above proof \emph{as is} by replacing $(\eta_n^X)_{n\in\Nstar}$ with $(\eta_{k_{m_n}}^X)_{n\in\Nstar}$.
    Indeed, the assumption that $X$ is \ac{EWC} \ac{as} only appears in two places: when applying Theorem~\ref{theorem:prokhorov for random measures}, and to establish the existence of $n_\epsilon$ such that \eqref{eq:assumptions wp 1-eps} hold.
    Both of these results only rely on the fact that $(\eta_n^X)_{n\in\Nstar}$ converges weakly to $P$ \ac{as}, which holds for $(\eta_{k_{m_n}}^X)_{n\in\Nstar}$ as well by construction of $\ell$.
    Overall, this shows that \eqref{eq:bound sup 10epsilon} holds for this specific sub-subsequence $(\eta_{k_{m_n}}^X)_{n\in\Nstar}$, and thus $\lim_{n\to\infty}\eta_{k_{m_n}}^X\phi = P\phi$ \ac{as} by (20.6) in \cite{Bau2011}.
    In other words, every subsequence of $(\eta_n^X\phi)_{n\in\Nstar}$ has itself a subsequence that converges to $P\phi$ \ac{as}, which is equivalent to convergence in probability of $(\eta_n^X\phi)_{n\in\Nstar}$ to $P\phi$ and concludes the proof.
\end{proof}
\begin{proof}[Proof of Corollary~\ref{clry:uniform averages of Cb(H)}]
    The proof is an identical repetition of the proofs of the implications \ref{stmt:ewc as:iii}$\implies$\ref{stmt:ewc as:ii} of Theorems~\ref{thm:characterization ewc as} and~\ref{thm:characterization ewc prob}, up to the replacement of the dense family of $\Cb(\X;\R)$ by a dense family of $\F$ and of absolute values by $\norm{\cdot}_\H$.
\end{proof}

\subsection{Proofs for Section~\ref{sec:connections}}
\begin{proof}[Proof of Lemma~\ref{lemma:WC but non EWC}]
    It is clear that the process $X$ converges weakly.
    We show that it is not \ac{EWC} in probability.
    To this end, we show that \ref{stmt:ewc prob:iii} in Theorem~\ref{thm:characterization ewc prob} does not hold.
    Let $f\in\Cb(\{0,1\})$, and assume that $u_n = \eta^X_n f$ converges in probability to some real-valued random variable $\ell$.
    We show that $f$ is constant.
    Let $n_p = 10^{p}-1$ for all $p\in\N$.
    Since $X$ is constant on index intervals of the form $[n_q+1, n_{q+1}]\cap\N$, grouping the terms in the sum on these intervals yields \begin{align*}
        u_{n_p} &= \frac{1}{n_p}\sum_{q=0}^{p-1}(n_{q+1} - n_q)f(X_{n_q+1})\\
            &= \frac{1}{n_p}\sum_{q=0}^{p-1}(n_{q+1}-n_q)\cdot\big(\ind_{2\N}(q)f(X_1) + \ind_{2\N+1}(q)f(1-X_1)\big),
    \end{align*}
    where $\ind_\S$ is the indicator function of a set $\S$.
    Now, let $v_p = u_{n_{p+1}} - u_{n_p}$.
    For an even value of $p$, we have \begin{align*}
        v_p &= \left(\frac{n_{p}}{n_{p+1}} - 1\right)u_{n_{p}} + \frac{n_{p+1}-n_p}{n_{p+1}}f(X_1)\\
            &= \frac{n_{p+1}-n_p}{n_{p+1}}\left(f(X_1) - u_{n_p}\right).
    \end{align*}
    For an odd value of $p$, similar calculations reveal \begin{equation*}
        v_p = \frac{n_{p+1}-n_p}{n_{p+1}}\left(f(1-X_1) - u_{n_p}\right).
    \end{equation*}
    We now use the facts that $v_p\to0$ in probability as $n\to\infty$ (by convergence in probability of $u$ and the definition of $v_p$) and $\frac{n_{p+1}-n_p}{n_{p+1}}\to\frac{9}{10}$ as $p\to\infty$ to conclude that both of the following hold \Pas, by \Pas-uniqueness of the limit of convergence in probability:\begin{align*}
        \frac{9}{10}\left(f(X_1) - \ell\right) &= 0,\Pas\\
        \frac{9}{10}\left(f(1-X_1) - \ell\right) &= 0,\Pas
    \end{align*}
    Therefore, $f(X_1) = f(1-X_1)$, \Pas, which shows that $f$ is constant.
    We can now conclude the proof: if $f$ is not constant, then the contraposition of the above shows that $\eta^X_nf$ does not converge in probability to any random variable.
    Since, by Theorem~\ref{thm:characterization ewc prob}, this is a necessary condition for \ac{EWC} in probability, $X$ is not \ac{EWC} in probability.
\end{proof}

\begin{proof}[Proof of Lemma~\ref{lemma:llne implies EWC}]
    The result follows immediately from Lemma~2.5 in \cite{SHS2009} together with Theorems~\ref{thm:characterization ewc as} and~\ref{thm:characterization ewc prob}.
    Indeed, assume that $X$ satisfies the \ac{WLLNE}, and let $P\in\Pcal(\X)$ be its asymptotic mean (which exists by Theorem 2.4 in \cite{SHS2009}). Then, for all $f$ in $\L_\infty(\X;\R)$, the following convergence in probability holds:\begin{equation}\label{eq:convergence of expectation for single f}
        Pf = \lim_{n\to\infty} \eta_n^X f,
    \end{equation}
    from Lemma 2.5 in the above reference.
    In particular, $\Cb(\X;\R)\subset\L_\infty(\X;\R)$, showing that \ref{stmt:ewc prob:iii} in Theorem~\ref{thm:characterization ewc prob} holds.
    Since $\X$ is separable, we deduce that $X$ is \ac{EWC} with limit measure $P$.
    Further, if $X$ satisfies the \ac{SLLNE}, then Lemma~2.5 in the same reference states that \eqref{eq:convergence of expectation for single f} holds where the convergence is now almost sure.
    The argument $\Cb(\X;\R)\subset\L_\infty(\X;\R)$ now enables to conclude that \ref{stmt:ewc as:iii} in Theorem~\ref{thm:characterization ewc as} holds, and thus $X$ is \ac{EWC} \ac{as}
\end{proof}

\begin{proof}[Proof of Proposition~\ref{prop:summary notions}]
    The implication that \ac{ESC} implies \ac{EWC} with the same limit measure follows trivially from the fact that the topology of convergence in total variation is finer than that of weak convergence.
    Finally, we only show that \ac{EWC} in probability implies weak \ac{AMS}, since the case for \ac{EWC} \ac{as} follows, and the proof that \ac{ESC} implies \ac{AMS} is similar by leveraging a result similar to \ref{stmt:ewc prob:i}$\implies$\ref{stmt:ewc prob:iii} in Theorem~\ref{thm:characterization ewc prob} for \ac{ESC} processes involving $\L_\infty(\X;\R)$ instead of $\Cb(\X;\R)$, and whose proof is left to the reader.

    Let $f\in\Cb(\X;\R)$.
    For any $n\in\Nstar$, we have \begin{align*}
        \left\lvert\frac1n\sum_{i=1}^N \E[f(X_i)] - (\E P)f\right\rvert 
            = \left\lvert\E[\eta_n^Xf - Pf]\right\rvert
            \leq \E[\lvert \eta_n^Xf - Pf\rvert],
    \end{align*}
    where the identity $(\E P)f = \E[Pf]$ follows from usual properties of the intensity measure.
    Now, from Theorem~\ref{thm:characterization ewc prob}, we have \begin{equation}\label{eq:convergence in probability of empirical averages}
        \lim_{n\to\infty}\eta_n^X f = Pf,
    \end{equation}
    where convergence is in probability.
    Furthermore, the sequence $(\eta_n^Xf)_{n\in\Nstar}$ is uniformly integrable by boundedness of $f$; indeed, we have $\E[\lvert\eta_n^Xf\rvert] \leq \norm{f}_\infty$ for all $n\in\Nstar$.
    It thus follows from known results on convergence of random variables that \eqref{eq:convergence in probability of empirical averages} is equivalent to convergence of $(\eta_n^Xf)_{n\in\Nstar}$ to $Pf$ in $L_1(\Omega,\P;\R)$.
    Hence, $\E[\lvert \eta_n^X f - Pf\rvert]\to 0$ as $n\to\infty$ and the result follows.
\end{proof}

\subsection{Proofs for Section~\ref{sec:stability}}
\begin{proof}[Proof of Theorem~\ref{thm:ewc transition pairs}]
    We begin with the direct implication and focus on the proof for $X$, as that for $Y$ is identical.
    Let $\pi_\X:(x,y)\in\X\times\Y\mapsto x\in\X$.
    Clearly, $\pi_\X$ is continuous.
    Now, if $Z$ is \ac{EWC} \ac{as}, it follows that \ref{stmt:ewc as:ii} in Theorem~\ref{thm:characterization ewc as} holds for $Z$, and in particular with $J$ the limit measure of $Z$\begin{equation*}
        \P[\forall f\in\Cb(\X;\R),~\limsup_n\lvert \eta_n^Z(f\circ \pi_\X) - J(f\circ \pi_\X)\rvert = 0] = 1,
    \end{equation*}
    because $f\circ \pi_\X\in\Cb(\X\times\Y;\R)$ for any $f\in\Cb(\X;\R)$.
    Furthermore, $\eta_n^Z(f\circ \pi_\X) = \eta_n^Xf$ and $J(f\circ\pi_\X) = Pf$, where $P$ is the random measure defined as $P[A] = J[A\times\Y]$ for all $A\in\B(\X)$.
    This shows that \begin{equation*}
        \P[\forall f\in\Cb(\X;\R),~\limsup_n\lvert \eta_n^Xf - Pf\rvert = 0] = 1,
    \end{equation*}
    and shows that $X$ is \ac{EWC} \ac{as} by Theorem~\ref{thm:characterization ewc as}.
    If $Z$ is only \ac{EWC} in probability, then the exact same reasoning as above but using Theorem~\ref{thm:characterization ewc prob} instead of Theorem~\ref{thm:characterization ewc as} shows that $X$ is \ac{EWC} in probability.

    We now assume that $\X$ and $\Y$ are compact and show the converse implication.
    We begin with the case where $X$ is \ac{EWC} \ac{as}~
    Let $f\in\Cb(\X\times\Y;\R)$ be fixed, and define for all $n\in\Nstar$ $\Delta_n = f(X_n,Y_n) - \E[f(X_n,Y_n)\mid X_{1:n},Y_{1:n-1}]$, where we recall the convention that conditioning on $Y_{1:0}$ means conditioning on the trivial $\sigma$-algebra $\{\emptyset,\Omega\}$.
    It satisfies $\E[\Delta_1] = 0$ and, for all $n\geq2$, $\E[\Delta_n\mid X_{1:n},Y_{1:n-1}] = 0$ \ac{as}~
    Further, the sequence $(\Delta_n)_{n\in\Nstar}$ is uniformly bounded by $2\cdot\norm{f}_\infty$.
    Hence, it satisfies the assumptions of Theorem~1 in \cite{Csoe1968} with $b_n = n$ for all $n\in\Nstar$.
    We deduce from that same theorem that \begin{equation}\label{eq:convergence Sn - Tn}
        \lim_{n\to\infty}\frac1n\sum_{i=1}^n\Delta_i = \lim_{n\to\infty}\frac1n\sum_{i=1}^n f(X_i,Y_i) - \E[f(X_i,Y_i)\mid X_{1:i},Y_{1:i-1}] = 0,
    \end{equation}
    where convergence is \ac{as}~
    Next, we show that \begin{equation}\label{eq:convergence Tn to Jf}
        \lim_{n\to\infty}\frac1n\sum_{i=1}^n\E[f(X_i,Y_i)\mid X_{1:i},Y_{1:i-1}] = Jf,
    \end{equation}
    with probability $1$, where $J$ is defined in \eqref{eq:joint limit measure}.
    Define $g:x\in\X\mapsto \int_{\Y}f(x,y)p(x,\d y)$; we have $g\in\Cb(\X;\R)$ since $p$ is Feller continuous.
    It holds with probability $1$ that \begin{align*}
        \frac1n\sum_{i=1}^n \E[f(X_i,Y_i)\mid X_{1:i},Y_{1:i-1}] 
            &= \frac1n\sum_{i=1}^n \E[f(X_i,Y_i)\mid X_{i}]\\
            &= \eta_n^X g,
    \end{align*}
    and thus by \ref{stmt:ewc as:iii} in Theorem~\ref{thm:characterization ewc as} applied to $X$, which is \ac{EWC}, \begin{equation*}
        \lim_{n\to\infty} \frac1n\sum_{i=1}^n \E[f(X_i,Y_i)\mid X_{1:i},Y_{1:i-1}] = Pg,\quad\Pas
    \end{equation*}
    Furthermore, \begin{equation*}
        Pg = \int_{\X}g(x)\d P(x) = \int_\X\int_\Y f(x,y)p(x,\d y)\d P(x) = Jf,
    \end{equation*}
    showing that \eqref{eq:convergence Tn to Jf} holds.
    Combining this with \eqref{eq:convergence Sn - Tn} shows \begin{equation}\label{eq:convergence to Jf}
        \lim_{n\to\infty} \eta_n^Z f = Jf,
    \end{equation}
    where the convergence is \ac{as}~
    The conclusion follows by leveraging Theorem~\ref{thm:characterization ewc as}, since $\X\times\Y$ is compact.

    Finally, the case where $X$ is only \ac{EWC} in probability follows by repeating the above steps by leveraging Theorem~\ref{thm:characterization ewc prob} instead of Theorem~\ref{thm:characterization ewc as}.
    The only difference is that the convergence in \eqref{eq:convergence Tn to Jf} is now in probability, and thus so is the one in \eqref{eq:convergence to Jf}, concluding the proof.
\end{proof}

\subsection{Proofs for Section~\ref{sec:consistency kme}}
\begin{proof}[Proof of Theorem~\ref{thm:consistency kme}]
    We begin with the following observation: from (ii) in Theorem 21 in \cite{SGF2010}, there exists a constant $C>0$ such that for all $Q_1$ and $Q_2$ in $\Pcal(\X)$, \begin{equation*}
        \mmd(Q_1,Q_2) \leq C\cdot \beta(Q_1,Q_2),
    \end{equation*}
    where $\beta$ is the Dudley metric and is known to metrize weak convergence on $\Pcal(\X)$ (see the discussion above (28) in the same reference).
    We emphasize that, although (ii) in Theorem 21 requires an additional separability assumption, that assumption is only used in the proof to show the upper bound on the Dudley metric; therefore, the lower bound still holds without the assumption.
    Consequently, if $(Q_n)_{n\in\Nstar}\subset\Pcal(\X)$ converges weakly to $Q\in\Pcal(\X)$, then $\mmd(Q_n,Q)\to0$ as $n\to\infty$.
    
    Let us now assume that $X$ is \ac{EWC} in probability with limit measure $P$.
    We use the same argument as in the end of the proof of Theorem~\ref{thm:averages of Cb(H)} to deduce from Corollary 20.8 in \cite{Bau2011} that this is equivalent to the fact that, for any strictly increasing $\ell\subset\Nstar$, there exists another strictly increasing  $m\subset\Nstar$ such that $\eta_{\ell_{m_n}}^X\cvw P$ \ac{as}.
    But then, $\lim_{n\to\infty}\mmd(\eta_{\ell_{m_n}}^X,P) = 0$ by the above upper bound, where convergence is \ac{as}
    By definition of the \ac{MMD}, this is equivalent to \begin{equation*}
        \lim_{n\to\infty}\frac{1}{\ell_{m_n}}\sum_{i=1}^{\ell_{m_n}}k(\cdot,X_i) = \mu_P,
    \end{equation*}
    where convergence is in $\norm{\cdot}_\H$ and is \ac{as}~
    It then follows from the converse implication of Corollary 20.8 in \cite{Bau2011} applied to $\norm{\frac{1}{\ell_{m_n}}\sum_{i=1}^{\ell_{m_n}}k(\cdot,X_i) - \mu_P}_\H$ that \eqref{eq:convergence kme} holds, where convergence is in probability.
    Finally, if $X$ is \ac{EWC} \ac{as}, we can repeat the above proof without going to a subsequence, showing the result in that case as well.

    The converse implication follows immediately from Proposition~\ref{prop:uniqueness limit}.
    \end{proof}

\subsection{Proofs for Section~\ref{sec:consistency SVM}}

The general structure of the proof of Theorem~\ref{thm:consistency svm under ewc} is identical to that of Theorem 2.17 in~\cite{SHS2009} thanks to Theorem~\ref{thm:stability SVM solutions} and Theorem~\ref{thm:averages of Cb(H)} which extend to our case Lemmas~4.2 and~4.4 in the same reference.
Only the last step differs, since the randomness of the limit measure $J$ makes the convergence of $\Rcal_{L,J}(f_{J,\lambda})$ to $\Rcal_{L,J,\H}$ as $\lambda\to0$ also a random event, and that convergence may not occur uniformly on (a full probability subset of) $\Omega$.
\begin{proof}[Proof of Theorem\,\ref{thm:consistency svm under ewc}]
    We begin with the case where $Z$ is \ac{EWC} \ac{as}~
    First, since $L$ is locally bounded, the function $L(\cdot,\cdot,0)$ is bounded, and thus we can assume without loss of generality that $\Rcal_{L,I}(0)\leq 1$ for all $I\in\Pcal(\X\times\Y)$.
    By Lemma~\ref{lemma:existence uniqueness SVM solutions}, we thus have $\norm{f_{I,\lambda}}_\H\leq\lambda^{-1/2}$ for all $\lambda > 0$.
    Furthermore, since $K$ is bounded, we can also assume without loss of generality that $\norm{K}_{\infty}\leq 1$ so that $\norm{f}_\infty\leq\norm{f}_\H$ for all $f\in\H$.
    Finally, by completeness of $(\Omega,\A)$ and measurability of $\Rcal_{L,J,\H}$ (cf. Corollary~\ref{clry:measurability optimal risk}), the assumption that $\H$ is $(L,J)$-rich $\Pas$ implies that $\Rcal_{L,J}^\star$ is measurable and is \ac{as} equal to $\Rcal_{L,J,\H}$.
    We denote by $\Omega_1$ a full probability set such that $\Rcal_{L,J,\H} = \Rcal_{L,J}^\star$.

    We have for all $n\in\Nstar$ and $\lambda>0$ the upper bound \begin{equation}\label{eq:upper bound excess risk}\begin{split}
        \left\lvert\Rcal_{L,J}(f_{\eta_n^Z,\lambda}) - \Rcal_{L,J}^\star\right\rvert
            &\leq \left\lvert\Rcal_{L,J}(f_{\eta_n^Z,\lambda}) - \Rcal_{L,J}(f_{J,\lambda})\right\rvert + \left\lvert\Rcal_{L,J}(f_{J,\lambda}) - \Rcal_{L,J}^\star\right\rvert\\
            &= \left\lvert\Rcal_{L,J}(f_{\eta_n^Z,\lambda}) - \Rcal_{L,J}(f_{J,\lambda})\right\rvert + \left\lvert\Rcal_{L,J}(f_{J,\lambda}) - \Rcal_{L,J,\H}\right\rvert\\
            &\leq \left\lvert L\right\rvert_{\lambda^{-1/2},1}\cdot\norm{f_{\eta_n^Z,\lambda} - f_{J,\lambda}}_\infty  + \left\lvert\Rcal_{L,J}(f_{J,\lambda}) - \Rcal_{L,J,\H}\right\rvert\\
            &\leq \frac{\left\lvert L\right\rvert_{\lambda^{-1/2},1}}{\lambda}\cdot\norm{\eta_n^Z(\phi h_{J,\lambda}) - J(\phi h_{J,\lambda})}_\H  + \left\lvert\Rcal_{L,J}(f_{J,\lambda}) - \Rcal_{L,J,\H}\right\rvert,
    \end{split}\end{equation}
    where the second step only holds on $\Omega_1$, the third one comes from local Lipschitz continuity and the integral definition of $\Rcal_{L,J}$, the last step comes from Theorem~\ref{thm:stability SVM solutions}, and $h_{J,\lambda}$ is as given in Corollary~\ref{clry:continuous representer}.
    We can now pick the regularization sequence $(\lambda_n)_{n\in\Nstar}$ as follows.
    Let $\epsilon>0$.
    First, note that by the expression of $h_{J,\lambda}$ in Corollary~\ref{clry:continuous representer} and by Lemma~\ref{lemma:measurability SVM}, the map $\omega\mapsto \Phi h_{J(\omega),\lambda}$ is measurable, and thus so are its integrals w.r.t. random measures (by Lemma~1.15.(i) in \cite{Kal2017}).
    Next, for all $n\in\Nstar$, and $\lambda>0$, let \begin{align*}
        F(\lambda,n) := \P\left[\sup_{m\geq n}\norm{\eta_m^Z(\phi h_{J,\lambda}) - J(\phi h_{J,\lambda})}_\H \geq \frac{\lambda}{\left\lvert L\right\rvert_{\lambda^{-1/2},1}}\epsilon\right].
    \end{align*}
    Now, $h_{J,\lambda}$ takes values in $\Cb(\X\times\Y;\G)$ by Corollary~\ref{clry:continuous representer}, and thus $\phi h_{J,\lambda}$ takes values in $\Cb(\X\times\Y;\H)$, which is separable by assumption.
    Therefore, by Corollary~\ref{clry:uniform averages of Cb(H)}, $\lim_{n\to\infty}\eta_n^Z(\phi h_{J,\lambda}) = J(\phi h_{J,\lambda})$, where convergence is \ac{as}~
    Consequently, it follows from Lemma 20.6 in \cite{Bau2011} that $\lim_{n\to\infty}F(\lambda,n) = 0$ for all $\lambda>0$.
    We now use Lemma~4.4 in \cite{SHS2009} to obtain a sequence $(\lambda_n)_{n\in\Nstar}$ such that $\lim_{n\to\infty}\lambda_n=0$ and $\lim_{n\to\infty}F(\lambda_n,n) = 0$.
    Let now $\delta>0$ be arbitrary, and take $n_0\in\Nstar$ such that for all $n\geq n_0$, both $F(n,\lambda_n)\leq \frac\delta2$ and $G_n\leq\frac\delta2$, where \begin{align*}
        G_n := \P\left[\sup_{m\geq n}\left\lvert\Rcal_{L,J}(f_{J,\lambda_m}) - \Rcal_{L,J,\H}\right\rvert \geq \epsilon\right].
    \end{align*}
    Such an $n_0$ exists since $\lim_{n\to\infty}G_n = 0$ by applying again Lemma 20.6 in \cite{Bau2011} to $\left\lvert\Rcal_{L,J}(f_{J,\lambda_n}) - \Rcal_{L,J,\H}\right\rvert$, which converges to $0$ \ac{as}~by a simple argument using Lemma~\ref{lemma:existence uniqueness SVM solutions}.
    It follows from \eqref{eq:upper bound excess risk} that for all $n\geq n_0$,
    \begin{align*}
        \P\left[\sup_{m\geq n}\left\lvert\Rcal_{L,J}(f_{\eta_m^Z,\lambda_m}) - \Rcal_{L,J}^\star\right\rvert \geq \epsilon\right]
            &\leq \P\left[
                    \sup_{m\geq n}\frac{\left\lvert L\right\rvert_{\lambda_m^{-1/2},1}}{\lambda_m}\cdot\norm{\eta_n^Z(\phi h_{J,\lambda_m}) - J(\phi h_{J,\lambda_m})}_\H \geq \epsilon
                ~\text{or}~\right.\\
                &\quad\qquad\left.
                \vphantom{\frac{\left\lvert L\right\rvert_{\lambda_m^{-1/2},1}}{\lambda_m}} 
                    \sup_{m\geq n}\left\lvert\Rcal_{L,J}(f_{J,\lambda_m}) - \Rcal_{L,J,\H}\right\rvert\geq \epsilon
            \right]\\
            &\leq F(\lambda_n,n) + G_n\\
            &\leq \delta.
    \end{align*}
    The converse implication of Lemma 20.6 in \cite{Bau2011} completes the proof of this case.
    The case when $Z$ is only \ac{EWC} in probability follows immediately by repeating the above steps without introducing the supremum $\sup_{m\geq n}$ in the definition of $F(\lambda,n)$ and of $G_n$, which concludes the proof.
\end{proof}
\section{Random compact sets}\label{apdx:random sets}
This section contains technical results necessary for the proof of Theorem~\ref{thm:averages of Cb(H)}.
The main result is Lemma~\ref{theorem:prokhorov for random measures}, which is a generalization of Prokhorov's theorem to finite or countable families of random measures.

\begin{definition}
    Let $(\X,d)$ be a metric space.
    We say that $\X$ has \emph{nice closed balls} if for any $x\in\X$ and $\epsilon>0$, the closed ball $B(x,\epsilon) = \{y\in\X\mid d(x,y)\leq \epsilon\}$ is either compact of equal to $\X$.
\end{definition}
\begin{lemma}
    Let $\X$ be a metrizable space. 
    Then, $\X$ has a compatible metric with nice closed balls if, and only if, $\X$ is locally compact.
\end{lemma}
\begin{proof}
    This follows immediately from Theorem 5.1.12 in \cite{Bee1993}.
\end{proof}

\begin{lemma}
    Let $\X$ be a Polish space and let $d$ be a metric that metrizes its topology.
    The set $\K(\X)$ of nonempty compact subsets of $\X$ can be equipped with the Hausdorff metric \begin{equation*}
        h:~(K_1, K_2)\in\K(\X)^2 \mapsto \max\left\{\sup_{x\in K_1}\inf_{y\in K_2} d(x,y), \sup_{x\in K_2}\inf_{y\in K_1} d(x,y)\right\}.
    \end{equation*}
    Then, $(\K(\X),h)$ is complete and separable.
    In particular, $\K(\X)$ equipped with the resulting topology is Polish.
\end{lemma}
\begin{proof}
    Completeness is proven in Theorem 3-3 in \cite{Hen1999}, and separability follows from Corollary~3.90 and Theorem~3.91 in \cite{AB2006}.
\end{proof}
In what follows, we assume without loss of generality that, if $\X$ is locally compact, then the Hausdorff distance $h$ on $\K(\X)$ is defined w.r.t. a compatible metric $d$ that has nice closed balls.

Before moving on to the main result, we show two useful preliminary results.
\begin{lemma}\label{lemma:nice closed dilations}
    Let $(\X,d)$ be a metric space with nice closed balls.
    Then, for any $K\in\K(\X)$ and $\epsilon>0$, the set \begin{equation*}
        \dil(K,\epsilon) = \{x\in\X\mid \inf_{y\in K}d(x,y) \leq \epsilon\} = \{x\in\X\mid d(x,K) \leq \epsilon\}
    \end{equation*}
    is either compact or equal to $\X$.
\end{lemma}
\begin{proof}
    Let $K\in\K(\X)$ and $\epsilon>0$.
    Let $(x_n)_{n\in\Nstar}\subset\dil(K,\epsilon)$.
    Take, for all $n\in\Nstar$, $z_n\in K$ such that $d(x_n, z_n)\leq \epsilon$; such a sequence exists by definition of $\dil(K,\epsilon)$.
    Since $(z_n)_{n\in\Nstar}$ is a sequence of the compact set $K$, it has a converging subsequence; say, $(z_{m_n})_{n\in\Nstar}$, and denote its limit by $z$.
    But then, for every $\delta>0$, there exists $N\in\Nstar$ such that $x_{m_n}\in B(z,\epsilon+\delta)$ for all $n\geq N$.
    We conclude by discussing two cases.
    If $B(z,\epsilon+\delta)=\X$ for all $\delta>0$, then in particular \begin{equation*}
        B(z,\epsilon) = \bigcap_{n\in\Nstar} B(z,\epsilon+n^{-1}) = \bigcap_{n\in\Nstar}\X = \X.
    \end{equation*}
    Since $B(z,\epsilon)\subset \dil(K,\epsilon)$, we deduce that $\dil(K,\epsilon) = \X$.
    Otherwise, there exists $\delta>0$ such that $B(z,\epsilon+\delta)$ is a proper subset of $\X$.
    Then, $(x_{m_n})_{n\in\Nstar}$ takes values in the compact set $B(z,\epsilon+\delta)$ from some finite rank $N\in\Nstar$ on.
    The sequence $(x_n)_{n\in\Nstar}$ thus has a subsequence converging in $B(z,\epsilon+\delta)$.
    Furthermore, since $\dil(K,\epsilon)$ is closed, we deduce that it must contain the limit of $(x_n)_{n\in\Nstar}$.
    In other words, either $\dil(K,\epsilon)=\X$, or each sequence thereof has a converging subsequence, showing compactness and concluding the proof.
\end{proof}
\begin{lemma}\label{lemma:evaluation is usc}
    Let $\X$ be a Polish space and $P\in\Pcal(\X)$.
    The map $e_P: K\in\K(\X)\mapsto P(K)\in[0,1]$ is upper semi-continuous, that is, for every $K\in\K(\X)$ and sequence $(K_n)_{n\in\Nstar}$ such that $h(K_n, K)\to 0$ as $n\to\infty$, we have $\limsup_n e_P(K_n) \leq e_P(K)$.
\end{lemma}
\begin{proof}
    First of all, our definition of upper semi-continuity is indeed equivalent to the general notion since $\X$ is Polish; see for instance Lemma 7.13.b in \cite{BS1996} for details.
    Let $K$ and $(K_n)_{n\in\Nstar}$ be as described in the lemma.
    The result follows from the fact that $\limsup_n K_n \subset K$.
    Indeed, if $x\in\limsup_n K_n$, then there exists a strictly increasing sequence $j:\Nstar\to\Nstar$ such that $x\in K_{j_n}$ for all $n\in\Nstar$.
    But then, for all $n\in\Nstar$,\begin{equation*}
        d(x, K) = \inf_{y\in K}d(x,y) \leq \sup_{z\in K_{j_n}}\inf_{y\in K}d(z,y) \leq h(K_{j_n},K).
    \end{equation*}
    Since the \ac{rhs} goes to $0$ as $n\to\infty$, we deduce that $d(x,K) = 0$, and thus $x\in K$ since $K$ is closed, showing the inclusion.
    By the reverse Fatou's lemma, it follows that \begin{equation*}
        \limsup_n e_P(K_n) \leq e_P(\limsup_n K_n) \leq e_P(K),
    \end{equation*}
    concluding the proof.
\end{proof}

\begin{theorem}\label{theorem:prokhorov for random measures}
    Let $\X$ be a locally compact Polish space and $\Q = (P_n)_{n\in\Nstar}$ be a finite or countable collection of random measures on $\X$.
    Assume that the set $E=\{\Q\text{ is sequentially compact}\}$ is measurable and has probability $1$.
    Then, for all $\epsilon>0$, there exists a random compact subset $K_\epsilon$ of $\X$ such that for all $P\in\Q, P(K_\epsilon)\geq 1-\epsilon$, \ac{as}
\end{theorem}
\begin{proof}
    Let $L$ be an arbitrary element of $\K(\X)$ and define 
    \begin{align*}
        A: \omega\in\Omega \mapsto \begin{cases}
            \big\{K\in\K(\X)\mid \forall P\in\Q,\,P_{\omega}(K) \geq 1-\epsilon\big\},&\text{if }\omega\in E,\\
            \big\{L\big\},&\text{otherwise}.
        \end{cases}
    \end{align*}
    The map $A$ takes values in the set of subsets of $\K(\X)$.
    We show that there exists a measurable selection of $A$ by using the Kuratowksi-Ryll-Nardzewski theorem; see for instance Theorem\,6.9.3 in \cite{BR2007}.
    We need to show that $A$ is weakly measurable and takes values in the set of nonempty closed subsets of $\K(\X)$.
    First, by Prokhorov's theorem, $A(\omega)\neq \emptyset$.
    Furthermore, it is also closed.
    Indeed, let $(K_n)_{n\in\Nstar}\subset A(\omega)$, and assume that it converges to some $K\in\K(\X)$.
    For any $P\in\Q$, we have $P_\omega(K) \geq \limsup_n P_\omega(K_n) \geq 1-\epsilon$, where the first inequality is from Lemma~\ref{lemma:evaluation is usc} and the second one from the assumption on the sequence $(K_n)_{n\in\Nstar}$.
    Since $P\in\Q$ is arbitrary, this shows closedness of $A(\omega)$.
    Summarizing, $A$ takes values in the set of nonempty closed subsets of $\K(\X)$.

    We are thus left to show that $A$ is weakly measurable, that is, that for all $U$ open in $\K(\X)$, the set \begin{align*}
        \hat A(U) &= \{\omega\in\Omega\mid A(\omega)\cap U \neq \emptyset\}\\
            &= \{\omega\in E\mid \exists K\in U,\forall P\in\Q,\, P_\omega(K)\geq 1-\epsilon\}\cup\{\omega\in \Omega\setminus E\mid L\in U\},
    \end{align*}
    belongs to $\A$.
    Since the second element of the \ac{rhs} is either $\emptyset$ or $\Omega\setminus E$, which are both measurable, we focus on the first element, which we denote by $\bar A_P(U)$.
    We have \begin{equation*}
        \bar A_P(U) = \{\omega\in E\mid \exists K\in U,\inf_{P\in\Q} P_\omega(K)\geq 1-\epsilon\}.
    \end{equation*}
    It is clear that it is sufficient to consider the case where $U$ is an open ball of $\K(\X)$, i.e., \begin{equation*}
        U = \{K\in\K(\X)\mid h(K_0,K)<\eta\},
    \end{equation*}
    where $K_0\in\K(\X)$ and $\eta>0$ are arbitrary.
    Define for all $n\in\Nstar$\begin{equation*}
        F_n = \dil\left(K_0,\eta(1-n^{-1})\right) = \{x\in\X\mid d(x,K_0)\leq \eta(1-n^{-1})\}.
    \end{equation*}
    where $d$ is the metric on $\X$ used to define $h$.
    It is clear that $K_0\subset F_n$ and that the sequence $F$ is nondecreasing.
    We discuss two cases.
    First, if $F_n = \X$ for some $n\in\Nstar$, then $U\supset\{K\in\K(\X)\mid K_0\subset K\}$; indeed, then any compact set $K$ in that \ac{rhs} satisfies $K\subset F_n$, and thus \begin{align*}
        h(K,K_0) 
            &= \max\left\{\sup_{x\in K}\inf_{y\in K_0} d(x,y), \sup_{x\in K_0}\inf_{y\in K} d(x,y)\right\}\\
            &\leq \max\left\{\sup_{x\in F_n}\inf_{y\in K_0} d(x,y), 0\right\}\\
            &= h(F_n,K_0) \leq \eta(1-n^{-1}) < \eta,
    \end{align*}
    which shows $K\in U$.
    In that case, $\bar A_P(U) = E$ by Prokhorov's theorem and is thus measurable.
    Second, and more interestingly, assume that $F_n\neq \X$ for all $n\in\Nstar$.
    Then, $F_n$ is compact by Lemma~\ref{lemma:nice closed dilations}, and thus $F_n\in U$ since \begin{align*}
        h(F_n, K_0) = \max\left\{\sup_{x\in F_n}\inf_{y\in K_0} d(x,y),0\right\} = \sup_{x\in F_n} d(x,K_0) < \eta.
    \end{align*}
    Furthermore, for any $K\in U$, it holds that $K\subset F_n$ for some $n\in\Nstar$; take for instance any $n\geq \left(1-\frac{h(K,K_0)}{\eta}\right)^{-1}$.
    Consequently, we have the equality \begin{equation*}
        \bar A_P(U) = \{\omega\in\Omega\mid \exists n\in\Nstar, \inf_{P\in\Q}P_\omega(F_n) \geq 1-\epsilon\}.
    \end{equation*}
    Indeed, the inclusion of the \ac{lhs} in the \ac{rhs} follows from the fact that any $K\in U$ such that $\inf_{P\in\Q} P_\omega(K) \geq 1-\epsilon$ must be in $F_n$ for some $n\in\Nstar$ by what precedes, and the converse inclusion simply from the fact that $F_n\in U$.
    This shows that \begin{equation*}
        \bar A_P(U) = \bigcup_{n\in\Nstar}\{\omega\in\Omega\mid \inf_{P\in\Q}P(F_n)\geq 1-\epsilon\}.
    \end{equation*}
    Since $\Q$ is finite or countable and the function $\omega\mapsto P_\omega(F_n)$ is measurable for any $n\in\Nstar$ and $P\in\Q$, so is $\omega\mapsto\inf_{P\in\Q}P_\omega(F_n)$.
    Therefore, $\bar A_P(U)\in\A$ as the preimage of $[1-\epsilon,1]$ by a measurable map.
    Finally, Theorem 6.9.3 in \cite{BR2007} guarantees the existence of a measurable selection of $A$ and concludes the proof.
\end{proof}

\section{A General representer theorem with separable Hilbert output space}\label{apdx:slt separable hilbert}
\subsection{Preliminaries on Fréchet subdifferentials and convex optimization}
\begin{definition}[{\cite{Kruger2009}}]\label{def:frechet subdifferential}
    Let $\B$ be a Banach space, $F:\B\to\R$ a functional, and $x\in\B$.
    Let $\B^\prime$ be the dual space of $\B$, and $\scalar{u^\prime,u}$ be the evaluation of $u^\prime$ in $u$, for all $u^\prime\in\B^\prime$ and $u\in\B$.
    The Fréchet subdifferential of $F$ in $x$ is the set \begin{equation*}
        \partial F(x) = \left\{x^\prime\in \B^\prime\middle|\liminf_{y\to x}\frac{F(y) - F(x) - \scalar{x^\prime,\,y-x}}{\norm{y-x}_\B}\geq 0\right\}.
    \end{equation*}
\end{definition}
The Fréchet subdifferential may be empty in general, but the following well-known result characterizes it for cases relevant for us.
\begin{proposition}[{\cite{Kruger2009}}]\label{prop:characterization subdifferential}
    Under the same setting as Definition~\ref{def:frechet subdifferential}:
    \begin{enumerate}
        \item If $F$ is Fréchet differentiable in $x$, then $\partial F(x) = \{\nabla F(x)\}$;
        \item If $F$ is convex, then \begin{equation*}
            \partial F(x) = \{x^\prime\in\B^\prime\mid \forall y\in\B,~f(y) \geq f(x) + \scalar{x^\prime,\,y-x}\}.
        \end{equation*}
    \end{enumerate}
\end{proposition}
In other words the Fréchet subdifferential coincides with the convex subdifferential for convex functionals.
Finally, the following result extends usual ones on minimization of convex functions defined on finite-dimensional spaces.
\begin{theorem}\label{thm:minimization convex functional}
    Let $\B$ be a reflexive Banach space and $F:\B\to\R$ be a functional.
    Assume that $F$ is coercive\footnote{Recall that $F$ is coercive if $F(x)\to\infty$ as $\norm{x}_\B\to\infty$} and $\alpha$-convex\footnote{Recall that $F$ is $\alpha$-convex if the functional $F - \alpha\cdot\norm{\cdot}_\B^2$ is convex.} with $\alpha>0$.
    Then, there exists a unique $x\in\B$ such that \begin{equation*}
        F(x) = \inf_{y\in\B} F(y).
    \end{equation*}
    Furthermore, $x$ is the unique element of $\B$ that satisfies $0\in\partial F(x)$.
\end{theorem}
\begin{proof}
    This follows immediately from Theorem A.6.9 in \cite{SC2008}.
\end{proof}
Finally, we emphasize that Hilbert spaces are reflexive Banach spaces, as noted in~\cite{PM2022}.

\subsection{The representer theorem}
The results and proofs of this section follow closely the lines of the corresponding results of Chapter 5 in \cite{SC2008}.
Indeed, these results generalize readily to the case where $\G$ is separable Hilbert; the novelty and contribution here reside in making sure that all of the arguments of the reference still hold.

\begin{lemma}[Properties of the risk]\label{lemma:properties risk}
    Let $\X$ be a measurable space, $\G$ a separable Hilbert space, and $\Y$ a complete subset of $\G$.
    Let $J\in\Pcal(\X\times\Y)$ and $L$ be a continuous, $J$-integrable Nemitski loss.
    Let $P$ be the marginal of $J$ on $\X$.
    Then, the following statements hold:\begin{enumerate}
        \item for all uniformly bounded sequence $(f_n)_{n\in\Nstar}\subset\L_\infty(\X;\G)$ that converges to some $f\in\L_\infty(\X;\G)$ $P$-\ac{ae}, we have \begin{equation*}
            \lim_{n\to\infty}\Rcal_{L,J}(f_n) = \Rcal_{L,J}(f);
        \end{equation*}
        \item the map $\Rcal_{L,J}:L_\infty(\X,P;\G)\to[0,\infty)$ is well-defined and continuous. In particular, it only takes finite values;
        \item if, additionally, $L$ is a Nemitski loss of order $2$, then $\Rcal_{L,J}:L_2(\X,P;\G)\to[0,\infty)$ is well-defined and continuous.
    \end{enumerate}
\end{lemma}
\begin{proof}
    The proof is an identical repetition of the proof of Lemma 2.17 in \cite{SC2008} and is omitted for conciseness.
\end{proof}
\begin{lemma}[Existence and uniqueness of \ac{SVM} solutions]\label{lemma:existence uniqueness SVM solutions}
    Let $\X$ be a measurable space, $\G$ be a separable Hilbert space, $\Y$ a complete subset of $\G$, and $\H$ a $\G$-valued \ac{RKHS} of measurable functions on $\X$ with bounded kernel $K$.
    Let $L$ be a convex loss function and $J\in\Pcal(\X\times\Y)$ such that $L$ is a $J$-integrable Nemitski loss of order $2$.
    For all $\lambda>0$, there exists a unique $f_{J,\lambda}\in\H$ such that \begin{equation*}
        \Rcal_{L,J}(f_{J,\lambda}) + \lambda\cdot\norm{f_{J,\lambda}}_\H^2 = \inf_{f\in\H} \Rcal_{L,J}(f) + \lambda\cdot\norm{f}_\H^2.
    \end{equation*}
    Furthermore, $\norm{f_{J,\lambda}}_{\H}\leq\sqrt{\frac{\Rcal_{L,J}(0)}{\lambda}}$.
\end{lemma}
\begin{proof}
    By Lemma~\ref{lemma:properties risk}, $\Rcal_{L,J}:L_2(\X,P;\G)\to[0,\infty)$ is a well-defined functional.
    Since $K$ is bounded, the inclusion operator $\iota_{P}:\H\to L_2(\X,P;\G)$ is bounded, and thus $R_\H = \Rcal_{L,J}\circ\iota_P$ is also a well-defined functional defined on $\H$.
    Furthermore, convexity of $L$ implies that $R_\H$ is also convex, and thus $R_\H + \lambda\cdot\norm{\cdot}_\H$ is $\lambda$-convex and coercive.
    The existence and uniqueness of $f_{J,\lambda}$ then follows from Theorem~\ref{thm:minimization convex functional}.
    The second claim follows immediately from the first one since \begin{equation*}
        \Rcal_{L,J}(f_{J,\lambda}) + \lambda\cdot\norm{f_{J,\lambda}}_\H^2 \leq \Rcal_{L,J}(0) 
    \end{equation*}
    and $\Rcal_{L,J}(f_{J,\lambda})\geq 0$.
\end{proof}
\begin{lemma} \label{lemma:subdifferential risk}
    Under the same setting as Lemma~\ref{lemma:existence uniqueness SVM solutions}, let $R$ be the risk functional $\Rcal_{L,J}$ defined on $L_2:=L_2(\X\times\Y,J;\G)$, that is, for $f\in L_2$,\begin{equation*}
        R(f) = \int_{\X\times\Y}L(x,y,f(x,y))\d J(x,y).
    \end{equation*}
    Then, $R$ is finite everywhere on $L_2$, is convex, and for all $f\in L_2$,\begin{equation}\label{eq:subdifferential risk}
        \partial R(f) = \{h\in L_2\mid h(x,y)\in\partial L(x,y,f(x,y))\,\text{for }J\text{-almost-every }(x,y)\in\X\times\Y\}.
    \end{equation}
\end{lemma}
\begin{proof}
    Everywhere-finiteness and convexity of $R$ follow immediately from the assumptions on $L$.
    We show the expression of the subdifferential, and denote by $D$ the \ac{rhs} of \eqref{eq:subdifferential risk}.
    First note that $L_2$ is equal to its dual, and thus $\partial R(f)$ is indeed a subset of it.
    Next, let $h\in D$, and let $g\in L_2$ be arbitrary.
    By definition of $D$, it holds $J$-\ac{ae} that \begin{equation*}
        L(x,y,g(x,y)) \geq L(x,y,f(x,y)) + \scalar{h(x,y),\,g(x,y)-f(x,y)}_\G.
    \end{equation*}
    Integrating w.r.t. $J$ yields \begin{align*}
        R(g) = \int_{\X\times\Y}L(x,y,g(x,y)) \d J(x,y) &\geq \int_{\X\times\Y} L(x,y,f(x,y)) \d J(x,y) + \int_{\X\times\Y}\scalar{h(x,y),\,(g-f)(x,y)}_\G\d J(x,y)\\
        &= R(f) + \scalar{h,\,g-f}_{L_2},
    \end{align*}
    where the equality between the second terms of the sums in the last step comes from Corollary 1.3.13 in \cite{HVNVW2016} since $\G$ is Hilbert and thus has the Radon-Nikodym property.
    Since $g$ is arbitrary, this shows that $h\in \partial R(f)$ and the inclusion $D\subset \partial R(f)$.
    Conversely, let $h\in\partial R(f)$.
    We prove $h\in D$ by constructing a set based on evaluations of $h$ and showing that it has measure $0$.
    This requires being precise and distinguishing functions from their equivalence classes for the relation of $J$-\ac{ae} equality:
    in what follows, if $u\in L_2$, we denote by $\bar u$ one of its representations (and conversely).
    For all $(x,y)\in\X\times\Y$, define the sets $A(x,y)$ and $B$ as follows: \begin{align*}
        A(x,y) 
            &= \{t\in\G\mid L(x,y,t) < L(x,y,\bar f(x,y)) + \scalar{\bar h(x,y), t-\bar f(x,y)}_\G\}\\
        B &= \{(x,y)\in\X\times\Y\mid A(x,y)\neq\emptyset\}.
    \end{align*}
    It is clear from the definition of $A(x,y)$ that $B\in\B(X\times\Y)$.
    Assume $J[B]>0$, and take any measurable function $\bar g$ such that, for all $(x,y)\in\X\times\Y$, $\bar g(x,y)\in A(x,y)$ if $(x,y)\in B$, and $\bar g(x,y)=\bar f(x,y)$ otherwise.
    We have \begin{align*}
        R(g) &= \int_{\X\times\Y} L(x,y,\bar g(x,y)) \d J(x,y)\\ 
            &=  \int_{B^\complement} L(x,y,\bar f(x,y))\d J(x,y) + \int_{B} L(x,y,\bar g(x,y))\d J(x,y)\\
            &< \int_{B^\complement} L(x,y,\bar f(x,y))\d J(x,y) + \int_{B} L(x,y,\bar f(x,y))\d J(x,y) + \int_{B}\scalar{\bar h(x,y), \bar g(x,y) - \bar f(x,y)}\d J(x,y)\\
            &= \int_{\X\times\Y}L(x,y,\bar f(x,y))\d J(x,y) + \int_{\X\times\Y}\scalar{\bar h(x,y), \bar g(x,y) - \bar f(x,y)}\d J(x,y)\\
            &= R(f) + \scalar{h, g-f}_{L_2},
    \end{align*}
    where we used successively the definition of $R$, of $g$, the property that the strict inequality in the definition of $A(x,y)$ is preserved by integrating over the set $B$ since it has positive measure, the fact that $\bar g(x,y) - \bar f(x,y) = 0$ for $(x,y)\notin B$, and properties of the dual pairing $\scalar{\cdot,\cdot}_{L_2}$.
    The above resulting inequality contradicts the assumption that $h\in\partial R(f)$.
    We deduce that, for any choice of representations, the corresponding set $B$ has measure $0$.
    This shows $h\in D$ by contradiction, concluding the proof.
\end{proof}
\begin{theorem}[General representer theorem with separable Hilbert output space]\label{thm:general representer theorem}
    Let $\X$ be a measurable space, $\G$ a separable Hilbert space, $\Y$ a complete subset of $\G$, and $\H$ a separable $\G$-valued \ac{RKHS} of measurable functions on $\X$ with bounded kernel $K$ and canonical feature map $\Phi:x\in\X\to K(\cdot, x)\in\L(\G;\H)$.
    Let $L$ be a convex loss function and $J\in\Pcal(\X\times\Y)$ such that $L$ is a $J$-integrable Nemitski loss of order $2$.
    Then, for all $\lambda>0$, there exists $h_{J,\lambda}\in\L_2(\X\times\Y,J;\G)$ such that, for all $(x,y)\in\X\times\Y$,\begin{align*}
        h_{J,\lambda}(x,y) &\in\partial L(x,y,f_{J,\lambda}(x)),\\
        f_{J,\lambda} &= -\frac{1}{2\lambda}J(\Phi h_{J,\lambda}),
    \end{align*}
    where the notation $\Phi h_{J,\lambda}$ stands for the function $(x,y)\in\X\times\Y\mapsto \Phi(x) h_{J,\lambda}(x,y)\in\H$.
    If, additionally, $L$ is locally Lipschitz continuous, then any such $h_{J,\lambda}$ satisfies\begin{equation*}
        \norm{h_{J,\lambda}}_\infty\leq \lvert L\rvert_{B_\lambda,1},
    \end{equation*}
    where $B_\lambda = \sqrt{\frac{\Rcal_{J,\lambda}(0)}{\lambda}}$.
\end{theorem}
\begin{proof}
    We consider $R_{L_2}$ the risk functional defined on the Bochner space $L_2:=L_2(\X\times\Y,J;\G)$ as defined in Lemma~\ref{lemma:subdifferential risk}.
    By Lemma~\ref{lemma:subdifferential risk}, $R_{L_2}$ is finite and convex with subdifferential in $f\in L_2$ given by \begin{equation*}
        \partial R_{L_2}(f) = \{h\in L_2\mid h(x,y)\in\partial L(x,y,f(x,y))\,\text{for }J\text{-almost-every }(x,y)\in\X\times\Y\}.
    \end{equation*}

    By boundedness of $\Phi$, the inclusion operator $\iota:\H\to L_2$ defined as $(\iota f)(x,y) = f(x)$ for all $f\in\H$ and $(x,y)\in\X\times\Y$ is bounded and for all $h\in L_2$,\begin{equation*}
        \scalar{h,\iota f}_{L_2} = \int_{\X\times\Y}\scalar{h(x,y),f(x)}_\G\d J(x,y) = \scalar{\int_{\X\times\Y}K(\cdot,x)h(x,y)\d J(x,y),~f}_\H,
    \end{equation*}
    and thus the adjoint $\iota^\star$ of $\iota$ is $\iota^\star:h\in L_2\mapsto J(\Phi h)\in\H$.
    Furthermore, the restriction $R_\H$ of $\Rcal_{L,J}$ to $\H$ satisfies $R_\H = R_{L_2}\circ \iota$, and thus the chain rule for Fréchet subdifferentials yields $\partial R_\H(f) = \partial (R_{L_2}\circ \iota)(f) = \iota^\star \partial R_{L_2}(\iota f)$, for all $f\in\H$.
    In addition, $f\mapsto\norm{f}^2_\H$ is Fréchet differentiable with derivative $2f$ for all $f\in\H$.
    Therefore, the subdifferential of the regularized risk $R_{\H,\lambda}=R_\H + \lambda\cdot\norm{\cdot}_\H$ at $f\in\H$ is \begin{align*}
        \partial R_{\H,\lambda}(f) &= 2\lambda f + \partial R_\H(f) \\
            &= 2\lambda f + \{J(\Phi h)\mid h\in \L_2(\X\times\Y,J;\G)\,\land\,\forall (x,y)\in\X\times\Y,\,h(x,y)\in\partial L(x,y,f(x))\,\}
    \end{align*}
    Recall now that $R_{\H,\lambda}$ has a minimum at $f_{J,\lambda}$ by Lemma~\ref{lemma:existence uniqueness SVM solutions}; thus, $0\in\partial R_{\H,\lambda}(f_{J,\lambda})$,
    establishing the existence of a function $h\in\L_2(\X\times\Y,J;\G)$ satisfying $h(x,y) \in\partial L(x,y,f_{J,\lambda}(x))$ for all $(x,y)\in\X\times\Y$ and $f_{J,\lambda} = -\frac{1}{2\lambda}J(\Phi h)$.
    
    We now assume that $L$ is locally Lipschitz continuous and show the claimed bound on $\norm{h}_\infty$.
    For all $(x,y)\in\X\times\Y$, \begin{align*}
        \norm{h(x,y)}_\G
            &\leq \sup\{\norm{v}_\G\mid v\in\partial L(x,y,f_{J,\lambda}(x))\}\\
            &\leq \sup\left\{\norm{v}_\G\middle| v\in\bigcup_{\substack{t\in\G\\\norm{t}_\G\leq B_\lambda}}\partial L(x,y,t)\right\},
    \end{align*}
    where the second inequality comes from $\norm{f_{J,\lambda}}_\infty\leq\norm{f_{J,\lambda}}_\H\leq B_\lambda$, by Lemma~\ref{lemma:existence uniqueness SVM solutions}.
    By Proposition A.6.11 in \cite{SC2008}, it holds that for every $r>0$ and $(x,y,t)\in\X\times\Y\times\G$ such that $\norm{t}_\G\leq r$, every $v\in\partial L(x,y,t)$ satisfies $\norm{v}_\G\leq \lvert L\rvert_{r,1}$.
    This shows $\norm{h(x,y)}_\G\leq \lvert L\rvert_{B_\lambda, 1}$ for all $(x,y)\in\X\times\Y$, concluding the proof.
\end{proof}

\subsection{Stability of solutions}
\begin{theorem}[Stability]\label{thm:stability SVM solutions}
    Let $\X$ be a measurable space, $\G$ a separable Hilbert space, $\Y$ a complete subset of $\G$, and $\H$ a separable $\G$-valued \ac{RKHS} of measurable functions on $\X$ with bounded kernel $K$ and canonical feature map $\Phi:x\in\X\to K(\cdot, x)\in\L(\G;\H)$.
    Let $L$ be a convex loss function and $J\in\Pcal(\X\times\Y)$ such that $L$ is a $J$-integrable Nemitski loss of order $2$.
    For all $\lambda > 0$, there exists $h_{J,\lambda} \in\L_2(\X\times\Y,J;\G)$ such that for all $I\in\Pcal(\X\times\Y)$ for which $L$ is a $J$-integrable Nemitski loss of order $2$ and all $(x,y)\in\X\times\Y$,
    \begin{align}
        h_{J,\lambda}(x,y) &\in\partial L(x,y,f_{J,\lambda}(x)),\label{eq:representer in subdiff}\\
        f_{J,\lambda} &= -\frac{1}{2\lambda}J(h\cdot\Phi),\label{eq:representer and SVM}\\
        \norm{f_{J,\lambda} - f_{I,\lambda}}_\H &\leq \frac{1}{\lambda}\norm{J(\Phi h_{J,\lambda}) - I(\Phi h_{J,\lambda})}_{\H}.\label{eq:stability SVM solutions}
    \end{align}
    If, additionally, $L$ is locally Lipschitz continuous, then any such $h_{J,\lambda}$ satisfies\begin{equation}
        \norm{h_{J,\lambda}}_\infty\leq \lvert L\rvert_{B_\lambda,1},\label{eq:representer is bounded}
    \end{equation}
    where $B_\lambda = \sqrt{\frac{\Rcal_{J,\lambda}(0)}{\lambda}}$.
\end{theorem}
\begin{proof}
    Let $\lambda>0$.
    By Theorem~\ref{thm:general representer theorem}, there exists a function $h\in\L_2(\X\times\Y,J;\G)$ that satisfies \eqref{eq:representer in subdiff} and \eqref{eq:representer and SVM}, and \eqref{eq:representer is bounded} as well if $L$ is assumed to be locally Lipschitz continuous.
    We thus only need to show that \eqref{eq:stability SVM solutions} holds.
    For this, notice that for all $(x,y)\in\X\times\Y$,\begin{equation*}
        L(x,y,f_{I,\lambda}(x)) \geq L(x,y,f_{J,\lambda}(x)) + \scalar{h(x,y), f_{I,\lambda}(x) - f_{J,\lambda}(x)}_\G,
    \end{equation*}
    by \eqref{eq:representer in subdiff}. Integrating against $I$ and using the reproducing property thus yields \begin{align*}
        \Rcal_{L,I}(f_{I,\lambda}) \geq \Rcal_{L,I}(f_{J,\lambda}) + \scalar{I(\Phi h),f_{I,\lambda} - f_{J,\lambda}}_{\H}.
    \end{align*}
    Furthermore, \begin{equation*}
        2\lambda\scalar{f_{I,\lambda} - f_{J,\lambda},\,f_{J,\lambda}} + \lambda \norm{f_{I,\lambda} - f_{J,\lambda}}^2_\H = \lambda\norm{f_{I,\lambda}}^2_\H - \lambda\norm{f_{J,\lambda}}^2_\H,
    \end{equation*}
    by manipulation of properties of $\norm{\cdot}_\H^2$.
    Combining the last two equations shows \begin{align*}
        \scalar{f_{I,\lambda} - f_{J,\lambda},\,I(\Phi h) + 2\lambda f_{J,\lambda}}_{\H} + \lambda\norm{f_{I,\lambda} - f_{J,\lambda}}_\H^2 
            &\leq \Rcal_{L,I}(f_{I,\lambda}) + \lambda\norm{f_{I,\lambda}}^2_\H - (
                \Rcal_{L,I}(f_{J,\lambda}) + \lambda\norm{f_{J,\lambda}}^2_\H
            )
            &\leq 0,
    \end{align*}
    where the second inequality comes from Lemma~\ref{lemma:existence uniqueness SVM solutions} applied to $I$.
    Consequently, the representation \eqref{eq:representer and SVM} and the Cauchy-Schwarz inequality yield \begin{align*}
        \lambda\norm{f_{J,\lambda} - f_{I,\lambda}}^2_\H 
            &\leq \scalar{f_{J,\lambda} - f_{I,\lambda},\,I(\Phi h) - J(\Phi h)} \\
            &\leq \norm{f_{J,\lambda} - f_{I,\lambda}}_\H\cdot\norm{I(\Phi h) - J(\Phi h)}_\H,
    \end{align*}
    which immediately implies \eqref{eq:stability SVM solutions} and concludes the proof.
\end{proof}

\begin{corollary}\label{clry:continuous representer}
    With the same notations and assumptions as in Theorem~\ref{thm:stability SVM solutions}, assume additionally that $\X$ is a topological space, $\H$ consists of continuous functions, and that $L$ is continuously differentiable.
    Then, the function $h_{J,\lambda}$ is unique and in $\Cb(\X\times\Y;\G)$. Specifically, $ h_{J,\lambda}:(x,y)\mapsto \nabla L(x,y,f_{J,\lambda}(x))$.
\end{corollary}
\begin{proof}
    Since $\H$ consists of continuous functions, we have in particular that $f_{J,\lambda}\in\H$ is continuous.
    The result follows immediately from \eqref{eq:representer in subdiff}.
\end{proof}
\section{Measurability of risks with random measures and of \acp{SVM} with infinite-dimensional outputs} \label{apdx:measurability SVMs}
The results of this section generalize Lemmas~6.3 (measurability of risks) and~6.23 (measurability of \acp{SVM}) from \cite{SC2008} to our setting.
The generalization of the first result consists of Lemma~\ref{lemma:measurability risk} and of Corollary~\ref{clry:measurability optimal risk}.
The challenge compared to the reference is that the risk is now defined w.r.t. a random measure instead of a fixed one.
The generalization of the second result is in Lemma~\ref{lemma:measurability SVM}.
It extends the reference to \acp{SVM} with separable Hilbert output spaces; our proof follows closely the lines of the original one.
Finally, we conclude on measurability of the sets involved in Definition~\ref{def:consistency} with Corollary~\ref{clry:measurability events consistency}.

We first recall the definition of a measurable learning method.
\begin{definition}
    Let $(\X,\A_\X)$ be a nonempty measurable space, $\G$ be a separable Hilbert space, and $\Y\subset\G$ a nonempty, complete subset.
    We say that a learning method $\Lfrak$ on $\X\times\Y$ is \emph{measurable} if, for all $n\in\Nstar$, the map \begin{align*}
        (\X\times\Y)^n\times\X&\to\G\\
            (Z,x)&\mapsto f_{Z}(x)
    \end{align*}
    is measurable with respect to the universal completion of the product $\sigma$-algebra on $(\X\times\Y)\times\X$, where $f_Z$ is the decision function produced by $\Lfrak$ with data set $Z$.
\end{definition}
We refer to Lemma~A.3.3 in \cite{SC2008} for a definition of the universal completion of a $\sigma$-algebra.
Importantly, we chose in Section~\ref{sec:preliminaries:sets} to equip any topological space with its Borel $\sigma$-algebra.
The following lemma ensures that this is not problematic for Polish spaces, as product $\sigma$-algebras coincide with Borel ones.
\begin{lemma}
    Let $\X$ be a Polish space, $\G$ be a separable Hilbert space, and $\Y\subset\G$ a nonempty, complete subset.
    Then, for all $n\in\Nstar$, the product $\sigma$-algebra on $(\X\times\Y)^n\times\X$ is equal to $\B\left((\X\times\Y)^n\times\X\right)$.
    Further, their universal completions also coincide.
\end{lemma}
\begin{proof}
    From Theorem 14.8 in \cite{Kle2013}, we have $\B\left((\X\times\Y)^n\times\X\right) = \left(\B(\X)\otimes\B(\Y)\right)^{\otimes n}\otimes \B(\X)$ since $\X$ and $\Y$ are Polish, where $\otimes$ denotes the product $\sigma$-algebra.
    This shows the first part of the result.
    The second part follows immediately by definition of the universal completion.
\end{proof}

The next lemma ensures measurability of events involving the risk w.r.t. a random measure given a measurable learning method.
\begin{lemma}\label{lemma:measurability risk}
    Let $\X$ be a Polish space, $\G$ be a separable Hilbert space, and $\Y\subset\G$ a nonempty, complete subset.
    Let $\Lfrak$ be a measurable learning method and $f_\cdot$ its decision function.
    Let $Z$ be an $\X\times\Y$-valued process and $J$ a random measure on $\X\times\Y$.
    Then, for any loss function $L:\X\times\Y\times\G\to[0,\infty)$ and any $n\in\Nstar$, the map $\omega\in\Omega\mapsto \Rcal_{L,J(\omega)}(f_{Z_{1:n}(\omega)})$ is measurable.
\end{lemma}
\begin{proof}
    By measurability of $\Lfrak$ and of $L$, we obtain measurability of the map $g:(\omega,x,y)\in\Omega\times\X\times\Y\mapsto L(x,y,f_{Z_{1:n}(\omega)}(x))$ by composition.
    It follows from item (i) in Lemma~1.15 in \cite{Kal2017} that $\omega\in\Omega\mapsto J(\omega)g(\omega,\cdot,\cdot) =: \Rcal_{L,J(\omega)}(f_{Z_{1:n}(\omega)})$ is also measurable, concluding the proof.
\end{proof}
\begin{corollary}\label{clry:measurability optimal risk}
    Let $\X$ be a Polish space, $\G$ a separable Hilbert space, $\Y\subset\G$ a nonempty, complete subset, and $\H$ a separable $\G$-valued \ac{RKHS} of measurable functions on $\X$.
    Let $J$ be a random measure on $\X\times\Y$.
    Then, the map $\omega\mapsto \Rcal_{\H,L,J(\omega)}$ is measurable.
\end{corollary}
\begin{proof}
    Let $(f_n)_{n\in\Nstar}$ be a dense family of $\H$.
    We have $\Rcal_{\H,L,J(\omega)} = \inf_{n\in\Nstar} \Rcal_{L,J}(f_n)$.
    Since $\H$ consists of measurable functions, the trivial learning method $\Lfrak_n$ that returns $f_n$ no matter the incoming data set is measurable for all $n\in\Nstar$.
    Consequently, the map $\omega\mapsto \Rcal_{L,J(\omega)}(f_n)$ is measurable for all $n\in\Nstar$.
    The result follows since a countable infimum of measurable functions is measurable.
\end{proof}

We are thus left to show that \acp{SVM} are measurable.
For this, we adapt the proof of Lemma 6.23 in \cite{SC2008} to separable output spaces.
We begin with the following technical lemma, whose proof follows the lines of that of Lemma 2.11 in the same reference.
\begin{lemma}\label{lemma:measurability loss}
    Let $\X$ be a Polish space, $\G$ be a separable Hilbert space, and $\Y\subset\G$ a nonempty, complete subset.
    Let $L$ be a loss function on $\X\times\Y\times\G$ and $\F\subset\L_0(\X;\G)$ a subset equipped with a complete and separable metric $d$.
    Assume that $d$ dominates the pointwise convergence, that is, \begin{equation*}
        \lim_{n\to\infty} d(f_n,f) \implies \forall x\in\X,~\lim_{n\to\infty} f_n(x) = f(x),
    \end{equation*}
    for all $f\in\F$ and $(f_n)_{n\in\Nstar}\subset\F$.
    Then, the evaluation map $(f,x)\in\F\times\X\mapsto f(x)\in\G$ is measurable.
    In particular, the map $(x,y,f)\mapsto L(x,y,f(x))$ is measurable.
\end{lemma}
\begin{proof}
    Since $d$ dominates pointwise convergence, for any fixed $x\in\X$ the $\G$-valued map $f\mapsto f(x)$ defined on $\F$ is continuous.
    Furthermore, since $\F$ consists of measurable functions, the map $x\mapsto f(x)$ is measurable for any fixed $f\in\F$.
    The first assertion then follows from Lemma III.14 in \cite{CV1977}.
    The second assertion follows immediately by applying the first one to the Polish space $\X\times\Y$.
\end{proof}
\begin{lemma}\label{lemma:measurability SVM}
    Let $\X$ be a Polish space, $\G$ a separable Hilbert space, $\Y\subset\G$ a nonempty, complete subset, and $\H$ a separable $\G$-valued \ac{RKHS} of measurable functions on $\X$ with bounded kernel $K$.
    Let $L$ be a convex loss function on $\X\times\Y\times\G$.
    For all $\lambda>0$, the corresponding \ac{SVM} that produces the decision function $f_{Z,\lambda}$ for $Z\in(\X\times\Y)^n$ for some $n\in\Nstar$ is a measurable learning method.
\end{lemma}
\begin{proof}
    From the assumptions, $\H$ is a separable metric space of measurable functions from $\X$ to $\G$.
    Further, its metric dominates pointwise convergence.
    Indeed, for any sequences $(f_n)_{n\in\Nstar}\in\H$ and $(x_n)_{n\in\Nstar}\subset\X$ converging to $f\in\H$ and $x\in\X$, respectively, we have \begin{equation*}
        \norm{f_n(x) - f(x)}_\G = \scalar{f_n - f, K(\cdot,x) (f_n - f)}_\H \leq \norm{K(\cdot,x)}_{\L(\H;\G)}\norm{f_n-f}_\H,
    \end{equation*}
    and the \ac{rhs} goes to $0$ by assumption.
    By Lemma\,\ref{lemma:measurability loss}, the map $(x,y,f)\in\X\times\Y\times\H\mapsto L(x,y,f(x))$ is measurable.
    We deduce that the map \begin{align*}
        \varphi:(\X\times\Y)^n\times\H&\to[0,\infty)\\
            (Z,f) &\mapsto \Rcal_{L,\eta_n^Z}(f) + \lambda\norm{f}_\H^2
    \end{align*}
    is measurable.
    Next, we apply (iii) in Lemma A.3.18 in \cite{SC2008} with $F(Z) := \H$ for $Z\in(\X\times\Y)^n$ combined with Lemma~\ref{lemma:existence uniqueness SVM solutions} to show that the map $Z\mapsto f_{Z,\lambda}$ is measurable with respect to the universal completion of the product $\sigma$-algebra of $(\X\times\Y)^n$.
    Consequently, the map $(\X\times\Y)^n\times\X\to\H\times\X$ defined by $(Z,x)\mapsto (f_{Z,\lambda}, x)$ is measurable.
    Finally, we deduce from the first conclusion of Lemma~\ref{lemma:measurability loss} that the map $(\X\times\Y)^n\times\X\to\G$ defined by $(Z,x)\mapsto f_{Z,\lambda}(x)$ is measurable, concluding the proof.
\end{proof}

\begin{corollary}\label{clry:measurability events consistency}
    Let $\X$ be a Polish space, $\G$ a separable Hilbert space, $\Y\subset\G$ a nonempty, complete subset, and $\H$ a separable $\G$-valued \ac{RKHS} of measurable functions on $\X$ with bounded kernel $K$.
    Let $L$ be a convex loss function on $\X\times\Y\times\G$, $Z$ be an $\X\times\Y$-valued process, and $J$ be a random measure on $\X\times\Y$.
    For all $\lambda>0$, $\epsilon>0$, and $n\in\Nstar$ the set $\{\omega\in\Omega\mid \Rcal_{L,J(\omega)}(f_{Z_{1:n}(\omega),\lambda})\leq \Rcal_{\H,L,J(\omega)}+\epsilon\}$ is measurable.
\end{corollary}
\begin{proof}
    By Lemmas~\ref{lemma:measurability SVM} and~\ref{lemma:measurability risk}, the map $\omega\in\Omega\mapsto \Rcal_{L,J(\omega)}(f_{Z_{1:n}(\omega),\lambda})$ is measurable.
    By Corollary~\ref{clry:measurability optimal risk}, the map $\omega\in\Omega\mapsto \Rcal_{\H,L,J(\omega)}$ is also measurable.
    The result follows immediately.
\end{proof}

\end{document}